%% file: main.tex
\documentclass[final,onefignum,onetabnum]{siamonline250211}
\input{macros}
\hypersetup{
  pdftitle={On Maximum Clique and Maximum Consensus Formulations in Point Cloud Registration},
  pdfauthor={Abolfazl Babanazari and Kaveh Fathian},
  pdfsubject={Preprint with supplementary materials}
}
\makeatletter
\newcommand{\arxivsupplementtitle}[1]{%
  \par
  \begingroup
    \null
    \begin{center}\color{header1}
      {\bfseries\MakeUppercase{\siampretitle}\HLtitle#1\vspace{3.8pt}\par}%
      \color{gray}\rule{\textwidth}{4pt}%
    \end{center}
    \par\vskip .11in
  \endgroup
  \thispagestyle{plain}%
}
\makeatother
\let\bibcite\svbibcite

\title{On Maximum Clique and Maximum Consensus Formulations in Point Cloud Registration}
\headers{Maximum Clique vs. Maximum Consensus Registration}{A. Babanazari and K. Fathian}
\author{
Abolfazl Babanazari
\thanks{
Department of Computer Science, Colorado School of Mines
(\email{abolfazl\_babanazari@mines.edu}, \email{fathian@ariarobotics.com})
}
\and Kaveh Fathian\footnotemark[1]
}

\begin{document}
\maketitle

\begin{abstract}
Maximum clique and maximum consensus are widely used formulations for robust point cloud registration.
Maximum consensus seeks one transformation that aligns as many matches as possible, while maximum clique selects matches through pairwise distance tests.
Maximum clique is often used as a substitute for or an approximation to maximum consensus, with empirical evidence suggesting a close relationship.
However, how closely the two formulations agree is not fully understood.
We study this relationship theoretically.
We show that, at any fixed positive tolerance, a maximum clique can have arbitrarily large alignment error and be arbitrarily larger than a maximum consensus set.
When rotation is fixed, we derive sharp error bounds and tests that exactly characterize maximum consensus for the translation subproblem.
We also study group-$k$ consistency, which uses consistency of groups of $k$ matches, and determine the smallest group sizes that guarantee alignment error bounds independent of point cloud scale and geometry.
These are $d+1$ for rotations and translations and $d+2$ when reflections are allowed, in dimensions $d\ge2$.
Finally, our experiments on nearly 8,000 registration examples quantify the frequency of these limitations in practice and show how the derived bounds can be used to check whether a subset of matches can be aligned within a point tolerance.

\end{abstract}

\section{Introduction}
\begin{figure}[tp]
\centering
\definecolor{OpeningSource}{HTML}{397EBA}
\definecolor{OpeningTarget}{HTML}{E4A349}
\definecolor{OpeningInlierColor}{HTML}{007A62}
\definecolor{OpeningCliqueColor}{HTML}{C52F4B}
\definecolor{OpeningConsensusColor}{HTML}{6543A2}
\begin{minipage}{.485\linewidth}
\centering\footnotesize
\textcolor{OpeningSource}{$\bullet$} Source\enspace
\textcolor{OpeningTarget}{$\bullet$} Target\enspace
\textcolor{OpeningInlierColor}{$\bullet$} GT inliers: 14
\end{minipage}\hfill
\begin{minipage}{.485\linewidth}
\centering\footnotesize
\textcolor{OpeningCliqueColor}{$\bullet$} Max-clique: 47\quad
\textcolor{OpeningConsensusColor}{$\bullet$} Max-consensus: 39
\end{minipage}\par
\nointerlineskip
\begingroup
\captionsetup[subfloat]{farskip=0pt}
\subfloat[Initial correspondences\label{fig:opening-a}]{%
\begin{minipage}[b]{.485\linewidth}
\centering
\includegraphics[width=\linewidth,max width=\linewidth]{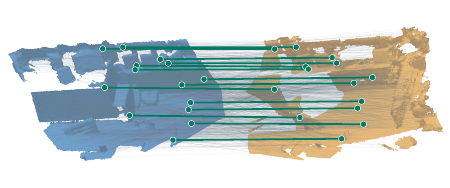}
\end{minipage}}%
\hfill
\subfloat[Sets under ground-truth alignment\label{fig:opening-b}]{%
\begin{minipage}[b]{.485\linewidth}
\centering
\includegraphics[width=\linewidth,max width=\linewidth]{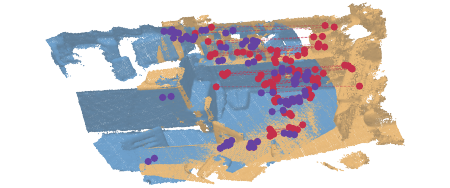}
\end{minipage}}%
\par
\endgroup
\smallskip
\subfloat[Maximum-clique fit: 55.0\,cm RMSE\label{fig:opening-c}]{%
\begin{minipage}[b]{.485\linewidth}
\centering
\includegraphics[width=\linewidth,max width=\linewidth]{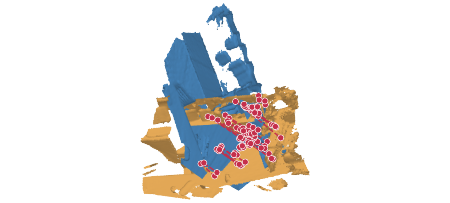}
\end{minipage}}%
\hfill
\subfloat[Maximum-consensus fit: 3.8\,cm RMSE\label{fig:opening-d}]{%
\begin{minipage}[b]{.485\linewidth}
\centering
\includegraphics[width=\linewidth,max width=\linewidth]{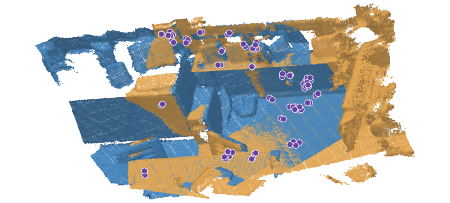}
\end{minipage}
}%
\caption{
Maximum clique and maximum consensus can yield different alignments.
(a) Tentative feature correspondences between two 3DMatch~\cite{zeng20173dmatch} scans using FPFH~\cite{rusu2009fast} and mutual nearest neighbor matching.
(b) Scans aligned using the ground-truth transformation, showing the maximum consensus set at tolerance $5$ cm and maximum clique of the pairwise consistency graph at $10$ cm.
(c) Fitting the maximum clique produces a poor registration.
(d) Fitting the maximum consensus set closely aligns the scans.
}
\label{fig:opening-correspondences}
\end{figure}

\begin{figure}[t]
\centering
\includegraphics[width=\textwidth]{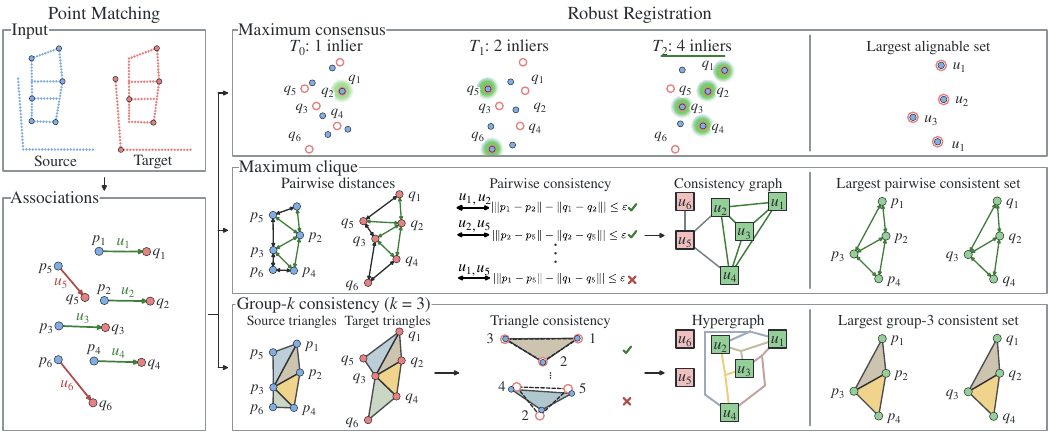}
\par
\caption{%
Point matching followed by correspondence selection.
Maximum consensus finds the transformation, $T_2$, that aligns the most matches within tolerance among possible transformations such as $T_0$, $T_1$ and $T_2$.
Maximum clique selects the largest set passing all pairwise distance tests.
Group-3 consistency requires every triple in the selected set to be alignable at that tolerance, possibly by different transformations.
Only some pairs and triangles are shown for clarity.
}
\label{fig:registration-pipeline}
\end{figure}
Point cloud registration estimates the alignment between measurements of a scene or object acquired from different viewpoints.
It arises in robotic mapping~\cite{yang2020teaser,mangelson2018pcm}, 3D reconstruction~\cite{arrigoni2016spectral}, microscopy~\cite{fortun2025multiview}, and many other applications.
When no initial alignment is available, registration pipelines often begin by matching local features across the point clouds, providing a set of tentative correspondences~\cite{zeng20173dmatch}.
Repeated structures, sensor noise, and limited overlap can make many of these matches incorrect, so registration must identify reliable matches for an alignment.

Maximum consensus and maximum clique are two widely used formulations for selecting reliable matches.
\emph{Maximum consensus} seeks one transformation that aligns as many correspondences as possible within a prescribed tolerance.
\emph{Maximum clique} selects the largest set that passes all pairwise distance tests, without estimating a transformation~\cite{parra2019practical}.
The two are closely related, and maximum clique is often used as a substitute for maximum consensus, but how far their solutions can differ is not fully understood.

At matched tolerances, maximum clique is a relaxation of maximum consensus.
Any set that one transformation aligns within the point tolerance passes every pairwise test and therefore forms a clique, so maximum clique size gives an upper bound on maximum consensus size.
The converse does not hold in general~\cite{forsgren2024gkcm,zhang20233d,laserna2025clireg}.
This upper bound alone does not tell us the clique's alignment error or how many of its matches can be aligned within the prescribed tolerance.
\Cref{fig:opening-correspondences} illustrates this distinction: a maximum clique and a maximum consensus set have similar sizes but few matches in common, and their fitted transformations produce very different registration results.

Testing larger groups offers a way to narrow this gap by detecting inconsistencies that pairwise tests miss.
In \emph{group-$k$ consistency}, every group of up to $k$ matches must admit an alignment within the prescribed tolerance, although different groups may use different transformations~\cite{forsgren2024gkcm}.
\Cref{fig:registration-pipeline} compares these three selection approaches on a small example.

In this work, we study the theoretical relationship between these three formulations, independently of a specific registration pipeline.
We begin with the simpler problem of known rotation, where pairwise tests bound the translation error and testing larger groups gives an exact characterization of maximum consensus.
When rotation must also be estimated, we show that a maximum clique can have arbitrarily large alignment error and be arbitrarily larger than a maximum consensus set, at any fixed positive tolerance.
We derive bounds that explain how this error depends on point cloud geometry.
We then show how group-$k$ tests give alignment error guarantees independent of point cloud scale and geometry, and determine the smallest group sizes needed.
In summary, the contributions of this paper are as follows:
\begin{itemize}
\item We show that maximum clique and maximum consensus methods for point cloud registration can give different results in general.
In particular, for certain point clouds, maximum clique can produce matches that cannot be obtained through maximum consensus at any fixed tolerance.

\item When the rotation that aligns two point clouds is fixed, we provide sharp lower and upper bounds on the maximum consensus solution size in terms of the maximum clique size.
In particular, in a $d$-dimensional space, we show that maximum clique using group-$(d+1)$ consistency is equivalent to maximum consensus.

\item In the general case with unknown rotation, we show that $d+1$ is the smallest group size for which maximum clique using group consistency can provide lower and upper bounds on the maximum consensus solution size.
When reflections are allowed, the smallest such group size is $d+2$.
Counterexamples establish that smaller groups cannot provide such guarantees.

\item We establish that, for every fixed dimension $d\ge2$, finding a maximum clique in a pairwise-consistency graph is NP-hard, whereas exact maximum consensus can be solved in polynomial time.

\item Experiments on nearly 8,000 registration pairs quantify how often these limitations occur in practice and show how the derived bounds can be used to check whether a selected set of matches can be aligned within a prescribed tolerance.
\end{itemize}

\section{Related Work}
\label{sec:related-work}
When correspondences are unknown, recovery guarantees have been established for matching using optimal mass transport~\cite{chen2013perfect} and a convex relaxation of Procrustes matching~\cite{dym2017exact}.
Given tentative correspondences, maximum consensus methods jointly estimate a transformation and select the correspondences that it fits within a tolerance.
Random sampling can approximate maximum consensus by fitting small subsets and counting their inliers~\cite{fischler1981random}.
Deterministic methods use branch-and-bound search to solve the maximum consensus problem~\cite{cai2019practical} or subproblems based on other robust losses~\cite{huang2024scalable}.
Geometric pruning methods can reduce this search by removing correspondences while preserving the maximum consensus solution~\cite{bustos2018gore}.

Clique-based outlier rejection offers practical runtimes and robustness to extreme outlier rates, with reported registration results including cases where over $99\%$ of matches are incorrect~\cite{parra2019practical,fathian2024clipperplus}.
Approaches include exact maximum clique search~\cite{parra2019practical}, approximations using large maximal cliques~\cite{fathian2024clipperplus}, and evaluating multiple candidate cliques before selecting an alignment~\cite{zhang20233d,laserna2025clireg}.
Graph structure and effective pruning can make maximum clique search practical despite its worst-case complexity~\cite{walteros2020easyclique,parra2019practical}.
Clique selection can precede robust fitting~\cite{yang2020teaser,shi2021robin} or maximum consensus in separate rotation, translation, and scale subproblems~\cite{gentner2023gmcr}.

The computational complexity of these formulations depends on the model and its dimension.
Linear maximum consensus is NP-hard when the parameter dimension is part of the input, but polynomial-time solvable when that dimension is fixed~\cite{chin2018robust}.
For maximum clique, geometric restrictions on the graph can also affect computational complexity~\cite{bonamy2021eptas}.
These results do not directly settle the comparison between rigid maximum consensus and maximum clique on registration consistency graphs, which we address here.

Extensions of pairwise consistency use triangles~\cite{zhou2016fast} or shared neighbors~\cite{chen2022sc2pcr} to reject matches that pass individual pairwise tests.
Group-$k$ consistency generalizes compatibility checks to $k$ measurements and formulates selection as maximum clique in a hypergraph~\cite{forsgren2024gkcm}.
These tests capture relations among larger groups that pairwise checks can miss.

Beyond point cloud registration, consistency-based selection is used in inter-robot loop-closure selection~\cite{mangelson2018pcm}, rotation averaging~\cite{shi2021robin}, and object-based localization~\cite{peterson2025roman}.
These applications share the question of whether locally compatible measurements support one common estimate.
A related geometric question arises in global graph rigidity, which asks when edge lengths uniquely determine a point configuration up to rigid motion~\cite{gortler2010globalrigidity}.

This work studies the theoretical relationship between maximum consensus, maximum clique, and group-$k$ consistency in correspondence-based point cloud registration.
We use results from classical geometry~\cite{henk1992generalization}, near-isometry theory~\cite{alestalo2001isometric,vaisala2001nearisometries}, and Procrustes analysis~\cite{ariasc_castro2020procrustes,maystre2025procrustes,lesk1998chebyshev} to establish when consistency tests guarantee a bound on alignment error and when they can fail.
We also draw on algorithms for polynomial sign patterns~\cite{arnon1984cylindrical,renegar1992} and geometric graph constructions~\cite{bonamy2021eptas} to compare the computational complexity of the two formulations.

\section{Preliminaries}
\label{sec:preliminaries}
Let $\{(\mathbf{p}_i,\mathbf{q}_i)\}_{i=1}^n$, with $\mathbf{p}_i,\mathbf{q}_i\in\mathbb{R}^d$, be a set of candidate correspondences.
A correspondence set $\mathcal{C}$ denotes a subset of the indices $\{1,\ldots,n\}$.
Registration seeks a rigid transformation $T(\mathbf{p})=\mathbf{R}\mathbf{p}+\mathbf{t}$, where $\mathbf{t}\in\mathbb{R}^d$ and $\mathbf{R}\in\mathrm O(d)$.
We write $T\in\mathrm E(d)$ for such transformations.
Pairwise tests cannot distinguish a configuration from its reflection, as illustrated in \cref{fig:reflection}.
We initially allow reflections and show that pairwise consistency can still fail under this more permissive model.
Proper rigid motions are treated separately in \cref{ssec:proper-motion}.
We use bold lowercase symbols for vectors, bold uppercase symbols for matrices, and nonbold symbols for real scalars.

\subsection{Maximum consensus}
Given a point-error tolerance $\tau\in\mathbb{R}^+$, maximum consensus seeks a transformation that aligns as many correspondences as possible within $\tau$.
To express this objective in terms of subsets, define the alignment error of a correspondence set $\mathcal{C}$ as
\begin{equation}
\rho(\mathcal{C})=\min_{T\in\mathrm{E}(d)}
\max_{i\in \mathcal{C}}\|T(\mathbf{p}_i)-\mathbf{q}_i\|.
\label{eq:rho}
\end{equation}
We call $\mathcal{C}$ \emph{alignable at tolerance $\tau$} if a single transformation makes every point-to-point error at most $\tau$, or equivalently $\rho(\mathcal{C})\le\tau$.
We call the largest such set the maximum consensus set and denote its size by
\begin{equation}
\mathrm{MC}(\tau)=\max\{|\mathcal{C}|:\rho(\mathcal{C})\le\tau\}.
\label{eq:consensus-as-consistency}
\end{equation}

\subsection{Pairwise consistency}
Pairwise consistency uses the fact that rigid transformations preserve distances.
Define the distance discrepancy of a pair and its maximum over a correspondence set
\begin{align}
\delta_{ij}&=\bigl|\,\|\mathbf{p}_i-\mathbf{p}_j\|-\|\mathbf{q}_i-\mathbf{q}_j\|\,\bigr|,
\label{eq:distance-discrepancy}\\
\delta(\mathcal{C})&=\max_{i,j\in \mathcal{C}}\delta_{ij}.
\label{eq:delta}
\end{align}

The consistency graph has one vertex per correspondence and an edge between $i$ and $j$ whenever $\delta_{ij}\le\epsilon$.
A clique is a set of vertices in which every pair is connected, and a maximum clique is the largest such set.
Let $\omega(\epsilon)$ denote the maximum clique size at tolerance $\epsilon$
\begin{equation}
\omega(\epsilon)=\max\{|\mathcal{C}|:\delta(\mathcal{C})\le\epsilon\}.
\label{eq:clique-as-consistency}
\end{equation}
Note that the edge tolerance $\epsilon$ bounds distance disagreement, whereas $\tau$ measures point error.

If $\mathcal{C}$ is alignable at tolerance $\tau$, the triangle inequality bounds the distance discrepancy of every pair by $2\tau$.
Consequently,
\begin{equation}
\delta(\mathcal{C})\le2\rho(\mathcal{C}),\qquad
\mathrm{MC}(\tau)\le\omega(2\tau).
\label{eq:necessary}
\end{equation}

\begin{figure}[tp]
\centering
\subfloat[Associations\label{fig:reflection-a}]{%
\begin{minipage}[b]{.24\linewidth}
\centering
\includegraphics[width=\linewidth]{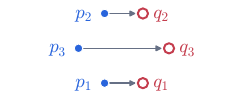}
\end{minipage}}%
\hfill
\subfloat[Pairwise distances\label{fig:reflection-b}]{%
\begin{minipage}[b]{.24\linewidth}
\centering
\includegraphics[width=\linewidth]{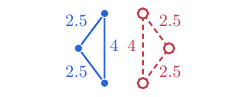}
\end{minipage}}%
\hfill
\subfloat[Max-clique fit\label{fig:reflection-c}]{%
\begin{minipage}[b]{.24\linewidth}
\centering
\includegraphics[width=\linewidth]{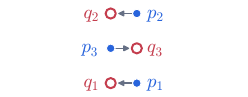}
\end{minipage}}%
\hfill
\subfloat[Max-consensus fit\label{fig:reflection-d}]{%
\begin{minipage}[b]{.24\linewidth}
\centering
\includegraphics[width=\linewidth]{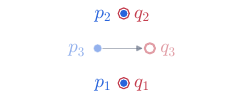}
\end{minipage}}%
\caption{
Two reflected triangles have identical pairwise distances, forming a clique at $\epsilon  = 0$.
At $\tau=0.5$, maximum consensus restricted to proper transformations can align only two matches.
}
\label{fig:reflection}
\end{figure}

\subsection{\texorpdfstring{Group-$k$}{Group-k} consistency}
For two correspondences, the alignment error is exactly half their distance discrepancy: $\rho(\{i,j\})=\frac{\delta_{ij}}{2}$.
A set therefore forms a clique at tolerance $2\tau$ if and only if every pair is alignable at tolerance $\tau$.
Group-$k$ consistency extends this requirement to larger subsets~\cite{forsgren2024gkcm}.
We call $\mathcal{C}$ \emph{group-$k$ consistent at tolerance $\tau$} if every subset of at most $k$ correspondences is alignable at $\tau$, allowing a different transformation for each subset.
Equivalently, $\rho_k(\mathcal{C})\le\tau$, where the \emph{group-$k$ consistency error} is the largest alignment error among these subsets
\begin{equation}
\rho_k(\mathcal{C})=\max_{\substack{\mathcal{I}\subseteq \mathcal{C}\\|\mathcal{I}|\le k}}\rho(\mathcal{I}).
\label{eq:full-hierarchy}
\end{equation}
Selecting the largest group-$k$ consistent set at tolerance $\tau$ is equivalent to finding a maximum clique in the corresponding consistency hypergraph, as illustrated for $k=3$ in \cref{fig:registration-pipeline}.
We denote this maximum clique size by
\begin{equation}
\Omega_k(\tau)=\max\{|\mathcal{C}|:\rho_k(\mathcal{C})\le\tau\}.
\label{eq:group-clique-size}
\end{equation}

Increasing $k$ strengthens the tests, giving
\begin{equation}
\frac{\delta(\mathcal{C})}{2} = \rho_2(\mathcal{C}) \le \rho_3(\mathcal{C}) \le \cdots \le \rho(\mathcal{C}).
\label{eq:full-hierarchy-bounds}
\end{equation}
Using larger $k$ can only decrease the largest set that passes them, so
\begin{equation}
\mathrm{MC}(\tau)\le\Omega_{k+1}(\tau)\le\Omega_k(\tau)\le\omega(2\tau),
\qquad k\ge2.
\label{eq:group-clique-hierarchy}
\end{equation}
If a uniform bound $\rho(\mathcal{C})\le c\rho_k(\mathcal{C})$ holds for every correspondence set $\mathcal{C}$ then
\begin{equation}
\Omega_k\!\left(\frac{\tau}{c}\right)\le\mathrm{MC}(\tau)\le\Omega_k(\tau).
\label{eq:group-consensus-bounds}
\end{equation}
Fitting every subset of at most $k$ correspondences within $\frac{\tau}{c}$ guarantees one common alignment within $\tau$, while every consensus set at $\tau$ passes the group-$k$ tests at that tolerance.
Therefore, alignment error bounds also give lower and upper bounds on maximum consensus size.

\section{Translation Subproblem}
\label{sec:translation}
We begin with the translation subproblem, in which rotation is fixed, and only translation is estimated.
This setting arises naturally in registration pipelines that estimate rotation and translation separately~\cite{yang2020teaser,gentner2023gmcr}, and provides a geometric baseline for understanding what local consistency can guarantee.
Even in this simpler setting, aligning every pair of correspondences within a tolerance does not ensure a common translation at that tolerance.
However, the gap is bounded, and we derive sharp alignment error bounds that progressively tighten as larger subsets are tested.

For a fixed rotation $\mathbf{R}$, each correspondence suggests the translation $\boldsymbol{\mu}_i=\mathbf{q}_i-\mathbf{R}\mathbf{p}_i$, illustrated by the arrows in \cref{fig:translation-a}.
Applying a translation $\mathbf{t}$ gives correspondence $i$ a point error of $\|\mathbf{t}-\boldsymbol{\mu}_i\|$.
Therefore, the translations that align correspondence $i$ within $\tau$ form a ball of radius $\tau$ centered at $\boldsymbol{\mu}_i$, shown in \cref{fig:translation-b}.
Two correspondences are alignable by a common translation within $\tau$ exactly when their disks intersect, or equivalently when $\|\boldsymbol{\mu}_i-\boldsymbol{\mu}_j\|\le2\tau$.
In the figure, every pair of disks intersects, but no translation belongs to all three, and aligning the entire set requires a larger tolerance.

To determine the tolerance required for a common translation, we define the translation-only alignment error as
\begin{equation}
\rho^{(t)}(\mathcal{C})=\min_{\mathbf{t}\in\mathbb{R}^d}\max_{i\in\mathcal{C}}\|\mathbf{t}-\boldsymbol{\mu}_i\|.
\label{eq:translation-radius}
\end{equation}
Geometrically, $\rho^{(t)}(\mathcal{C})$ is the radius of the smallest ball enclosing the candidate translations.
Jung's theorem~\cite{henk1992generalization} bounds this radius in terms of the largest distance between the candidates.
Since the pairwise test ensures that these distances are at most $2\tau$, Jung's theorem bounds the alignment error by
\begin{equation}
\rho^{(t)}(\mathcal{C})\le \tau J_d,
\qquad
J_d=\sqrt{\frac{2d}{d+1}}.
\label{eq:jung}
\end{equation}

In one dimension, $J_1=1$, so maximum clique under this pairwise test is the same as the maximum consensus.
The same equivalence holds for scalar scale estimation~\cite{gentner2023gmcr}.
In 2D and 3D, the pairwise test guarantees a common translation within approximately $1.15\tau$ and $1.22\tau$, respectively.

The generalized Jung's theorem~\cite{henk1992generalization} extends the pairwise bound to group-$k$ consistency tests.
Define $\rho_k^{(t)}$ as in~\cref{eq:full-hierarchy}, using $\rho^{(t)}$ in place of $\rho$.
The alignment error is then bounded by
\begin{equation}
\rho^{(t)}(\mathcal{C})\le \rho_k^{(t)}(\mathcal{C}) J_{d,k},
\quad
J_{d,k}=\sqrt{\frac{kd}{(d+1)(k-1)}},
\label{eq:translation-factor}
\end{equation}
for $2\le k\le d+1$.
At $k=d+1$, the factor equals one, so aligning every subset of at most $d+1$ correspondences within $\tau$ guarantees a common translation at the same tolerance.
Selecting the largest set satisfying these tests therefore solves translation-only maximum consensus exactly.
\begin{proposition}[Translation consensus bounds]
For translation subproblem, let $\mathrm{MC}^{(t)}(\tau)$ and $\Omega_k^{(t)}(\tau)$ denote the maximum consensus and maximum group-$k$ consistent set sizes, respectively.
For $\tau>0$ and $2\le k\le d+1$,
\begin{equation}
\Omega_k^{(t)}\!\left(\frac{\tau}{J_{d,k}}\right)
\le \mathrm{MC}^{(t)}(\tau)
\le \Omega_k^{(t)}(\tau).
\label{eq:translation-consensus-bounds}
\end{equation}
In particular,
\begin{equation}
\Omega_{d+1}^{(t)}(\tau)=\mathrm{MC}^{(t)}(\tau).
\label{eq:translation-consensus-equivalence}
\end{equation}
\end{proposition}
\begin{proof}
The result follows from \cref{eq:translation-factor,eq:group-consensus-bounds}, with $J_{d,d+1}=1$.
\end{proof}

\begin{figure}[tp]
\centering
\begin{minipage}[t]{.48\linewidth}
\vspace{0pt}
\begin{minipage}[c][1.65in][c]{\linewidth}
\centering
\subfloat[\label{fig:translation-a}]{\begin{minipage}[b]{.40\linewidth}
\centering
\includegraphics[width=\linewidth,height=1.4in,keepaspectratio]{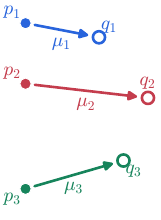}
\end{minipage}}\hfill
\subfloat[\label{fig:translation-b}]{\begin{minipage}[b]{.58\linewidth}
\centering
\includegraphics[width=\linewidth,height=1.4in,keepaspectratio]{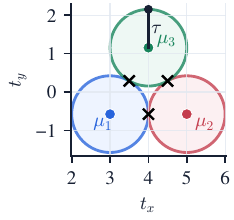}
\end{minipage}}
\end{minipage}
\caption{
In the translation subproblem, pairwise consistency does not guarantee a common translation at the same tolerance.
(a) Each arrow $\boldsymbol{\mu}_i$ is the translation vector that aligns source $\mathbf{p}_i$ and target $\mathbf{q}_i$.
(b) Every disk of radius $\tau$ centered at $\boldsymbol{\mu}_i$ contains all translations that align $\mathbf{p}_i$ and $\mathbf{q}_i$ within $\tau$.
}
\label{fig:translation}
\end{minipage}\hfill
\begin{minipage}[t]{.48\linewidth}
\vspace{0pt}
\begin{minipage}[c][1.65in][c]{\linewidth}
\centering
\subfloat[\label{fig:thin-a}]{\begin{minipage}[b]{73.190bp}
\centering
\includegraphics[width=\linewidth]{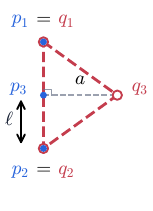}
\end{minipage}}\hspace{10pt}%
\subfloat[\label{fig:thin-b}]{\begin{minipage}[b]{102.448bp}
\centering
\includegraphics[width=\linewidth]{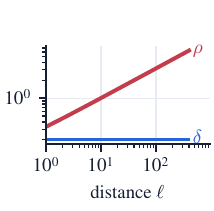}
\end{minipage}}
\end{minipage}
\caption{
Unbounded joint error at fixed pairwise discrepancy.
(a) Source and target points.
(b) Maximum pairwise distance error $\delta$ and alignment error $\rho$ as distance $\ell$ grows.
}
\label{fig:thin}
\end{minipage}
\end{figure}

\section{Pairwise consistency with unknown rotation}
\label{sec:full-problem}
We now consider the case where both rotation and translation are unknown.
We first show that pairwise consistency does not guarantee bounded alignment error, and that the maximum clique size can exceed the maximum consensus size by an arbitrarily large factor.
We then identify additional geometric conditions under which pairwise consistency guarantees bounded alignment error.

\subsection{Limitations of pairwise consistency}
We show that alignment error can grow arbitrarily large while pairwise consistency holds at a fixed tolerance.
To simplify the arguments, we exploit the following constraint that pairwise consistency does not capture.
\begin{lemma}[Zero-sum combination bound]
\label{lem:affine-relation}
For a correspondence set $\mathcal{C}$ and coefficients $\alpha_i$ satisfying $\sum_{i\in\mathcal{C}}\alpha_i=0$,
\begin{equation}
\bigl|
\|\textstyle\sum_i\alpha_i \mathbf{p}_i\|
-
\|\textstyle\sum_i\alpha_i \mathbf{q}_i\|
\bigr|
\le
\rho(\mathcal{C}) \sum_i|\alpha_i|.
\label{eq:affine-relation}
\end{equation}
\end{lemma}
\begin{proof}
If one transformation aligns all correspondences within $\tau$, it also aligns corresponding convex combinations within $\tau$ as well.
Differences of convex combinations give the zero-sum combinations used in~\eqref{eq:affine-relation}. \Cref{app:counterexamples} of the supplementary material gives the complete proof.
\end{proof}

\begin{example}[Unbounded alignment error]
\label{ex:unbounded-error}
Consider aligning the following point clouds at tolerance $\tau$, shown in \cref{fig:thin-a},
\begin{equation}
\begin{aligned}
\mathcal{P}=\{(0,\ell),(0,-\ell),(0,0)\},\qquad
\mathcal{Q}=\{(0,\ell),(0,-\ell),(a,0)\}.
\end{aligned}
\label{eq:thin}
\end{equation}
The point $\mathbf{p}_3$ is the midpoint of $\mathbf{p}_1$ and $\mathbf{p}_2$, whereas $\mathbf{q}_3$ lies a distance of $|a|$ from the midpoint of $\mathbf{q}_1$ and $\mathbf{q}_2$.
Therefore, \cref{lem:affine-relation}, with $\boldsymbol{\alpha}=(\frac{1}{2},\frac{1}{2},-1)$, gives $\rho(\mathcal{C})\ge\frac{|a|}{2}$.
Translating every source point horizontally by $\frac{a}{2}$ attains this bound, so $\rho(\mathcal{C})=\frac{|a|}{2}$.

The distance between the first two source points equals the distance between their corresponding target points, and the other two pairwise distances each change by $\delta=\sqrt{\ell^2+a^2}-\ell$.
To keep this discrepancy fixed, choose
\begin{equation}
a=\sqrt{(\ell+2\tau)^2-\ell^2}=2\sqrt{\tau(\ell+\tau)}.
\label{eq:thin-height}
\end{equation}
As shown in \cref{fig:thin-b}, this keeps $\delta=2\tau$, while
\begin{equation}
\lim_{\ell\to\infty}\rho(\mathcal{C})
=\lim_{\ell\to\infty}\frac{a}{2}=\infty.
\label{eq:thin-error}
\end{equation}
\end{example}

\begin{example}[Maximum clique can be arbitrarily larger than maximum consensus]
\label{ex:unbounded-size-ratio}
As illustrated in \cref{fig:parabola}, place $n$ source points on a line and their targets on a parabola
\begin{equation}
\begin{aligned}
x_i&=i-\frac{n+1}{2},\qquad i=1,\ldots,n,\\
\mathbf{p}_i&=(\ell x_i,0),\qquad \mathbf{q}_i=(\ell x_i,ax_i^2).
\end{aligned}
\label{eq:parabola}
\end{equation}

A distance calculation similar to that in \cref{ex:unbounded-error} gives
\begin{equation}
\delta_{ij}=|x_i-x_j|
\left(\sqrt{\ell^2+a^2(x_i+x_j)^2}-\ell\right)
\le
\frac{a^2|x_i-x_j|(x_i+x_j)^2}{2\ell}.
\label{eq:parabola-pairs}
\end{equation}
Choosing $\ell$ sufficiently large ensures $\delta_{ij}\le2\tau$ for every pair, forming a complete consistency graph with $\omega(2\tau)=n$.

For indices $i<j<k$, co-linearity of source points gives
\begin{equation}
\mathbf{p}_j=\lambda \mathbf{p}_i+(1-\lambda)\mathbf{p}_k,
\qquad
\lambda=\frac{k-j}{k-i}.
\label{eq:parabola-source-average}
\end{equation}

Using \cref{lem:affine-relation} with coefficients $\boldsymbol{\alpha}=(\lambda,-1,1-\lambda)$ gives
\begin{equation}
|a|(j-i)(k-j)
\leq
2\rho(\{i,j,k\}).
\label{eq:parabola-triple-error}
\end{equation}
Since $(j-i)(k-j) \ge 1$, the alignment error is at least $\frac{a}{2}$.
Choosing $a>2\tau$ makes every triple inconsistent.
Consequently
\begin{equation}
\omega(2\tau)=n,\qquad \mathrm{MC}(\tau)=2.
\label{eq:clique-consensus-ratio}
\end{equation}
\end{example}

\begin{proposition}[No uniform lower bound]
Fix $d\ge2$ and $\tau>0$. There are no constants $a,b>0$, independent of the
correspondences, such that
\begin{equation}
a\,\omega\!\left(\frac{\tau}{b}\right)
\le \mathrm{MC}(\tau)
\end{equation}
holds for every correspondence set in $\mathbb R^d$.
\end{proposition}
\begin{proof}
The claims follow from
\cref{ex:unbounded-error,ex:unbounded-size-ratio}.
Both constructions extend to every $d\ge2$ by appending
zero coordinates.
\end{proof}

\subsection{Alignment bounds using point cloud geometry}
Both previous counterexamples exploit increasing point cloud diameters to reduce pairwise distance discrepancies.
This dependence on scale is captured by classical near-isometry theory, which studies how closely a map that approximately preserves distances can be matched by a rigid transformation~\cite{alestalo2001isometric,vaisala2001nearisometries}.
Let $\mathcal{C}$ be pairwise consistent at tolerance $\epsilon$, and let $D$ be the diameter of its source points.
For $\epsilon\le D$,
\begin{equation}
\rho(\mathcal{C})\le c_d\sqrt{D\epsilon},
\label{eq:diameter-certificate}
\end{equation}
where $c_d$ depends only on dimension.
For $\epsilon\le0.01D$, the stated constants are $c_2=8.73$ in 2D and $c_3=18.8$ in 3D~\cite{alestalo2001isometric}.
The $\sqrt{D}$ dependence is also unavoidable without additional assumptions such as point cloud spread or thickness, as \cref{ex:unbounded-error} shows.

The diameter bound can be conservative because it uses a single global scale for all pairs.
We instead derive a bound that combines each pair's discrepancy with its own source and target distances.
Define the maximum squared-distance discrepancy of $\mathcal{C}$ as
\begin{equation}
h(\mathcal{C})=\max_{i,j\in\mathcal{C}}
\left|\|\mathbf{p}_i-\mathbf{p}_j\|^2
-\|\mathbf{q}_i-\mathbf{q}_j\|^2\right|.
\label{eq:squared-discrepancy}
\end{equation}
For a pair with source and target distances $u$ and $v$, we have $|u^2-v^2|=|u-v|(u+v)$.
Thus, squared-distance discrepancies weight each pair's distance disagreement by the sum of its source and target lengths, penalizing discrepancies on longer pairs more strongly.

\begin{proposition}[Squared-distance bound]
\label{prop:squared-distance}
For any correspondence set $\mathcal{C}$,
\begin{equation}
\rho(\mathcal{C})^2\le\sqrt{2d}\,h(\mathcal{C}).
\label{eq:squared-certificate}
\end{equation}
\end{proposition}
The proof is given in \cref{app:alignment-bounds}.

Consider a graph with one vertex per correspondence and an edge whenever the pairwise squared-distance discrepancy is at most $\eta$.
Let $\omega_h(\eta)$ denote its maximum clique size.
At $\eta=\tau^2/\sqrt{2d}$, \cref{prop:squared-distance} guarantees that every clique is alignable at tolerance $\tau$, giving
\begin{equation}
\omega_h\!\left(\frac{\tau^2}{\sqrt{2d}}\right)\le\mathrm{MC}(\tau).
\label{eq:squared-lower-bound}
\end{equation}

The same bound also yields a diameter-based guarantee.
Pairwise consistency bounds the target diameter by $D+\epsilon$, so $h(\mathcal{C})\le(2D+\epsilon)\epsilon$.
\Cref{prop:squared-distance} therefore gives $c_2=2.01$ and $c_3=2.22$ in~\cref{eq:diameter-certificate} in the same regime.

\begin{figure}[tp]
\centering
\begin{minipage}[t]{.48\linewidth}
\vspace{0pt}
\begin{minipage}[c][1.5in][c]{\linewidth}
\centering
\includegraphics[width=\linewidth]{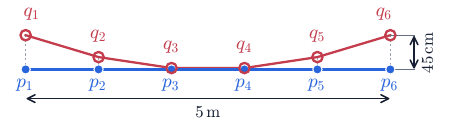}
\end{minipage}
\caption{
All six correspondences are pairwise consistent at a distance-discrepancy tolerance of $5$ cm, but no triple can be aligned within $2.5$ cm.
Thus maximum clique has size $6$, while maximum consensus has size $2$.
}
\label{fig:parabola}
\end{minipage}\hfill
\begin{minipage}[t]{.48\linewidth}
\vspace{0pt}
\begin{minipage}[c][1.5in][c]{\linewidth}
\centering
\subfloat[\label{fig:higher-order-a}]{\begin{minipage}[b]{.46\linewidth}
\centering
\includegraphics[width=\linewidth,height=1.25in,keepaspectratio]{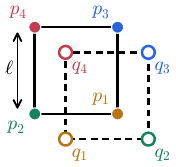}
\end{minipage}}\hfill
\subfloat[\label{fig:higher-order-b}]{\begin{minipage}[b]{.51\linewidth}
\centering
\includegraphics[width=\linewidth,height=1.25in,keepaspectratio]{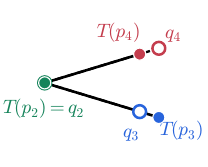}
\end{minipage}}
\end{minipage}
\caption{
Group-3 consistency does not bound the alignment error when reflections are allowed.
(a) Four correspondences forming a rectangle of width $2$ and height $\ell$, with the two lower targets exchanged.
(b) Every triple can be aligned with maximum point error at most $\sqrt{\ell^2+4}-\ell$.
}
\label{fig:higher-order}
\end{minipage}
\end{figure}

\section{\texorpdfstring{Group-$k$}{Group-k} consistency}
\label{sec:higher-order}
Group-$k$ consistency strengthens the local tests themselves and can provide alignment error guarantees without controlling point cloud scale or geometric conditioning.
We first show that group-$(d+1)$ consistency can still leave the alignment error unbounded.
We then prove that $d+2$ is the smallest group size that guarantees an alignment error bound.
Finally, we separately study proper rigid transformations in $\mathrm{SE}(d)$, where the consistency hierarchy behaves differently.

\subsection{Limits of testing smaller groups}
Testing larger groups does not immediately provide an alignment error guarantee.

\begin{example}[Group-3 consistency cannot bound the error]
\label{ex:planar-higher-order}
 Consider aligning two point clouds whose source points are
 \begin{equation}
\mathcal{P} = \{(1,0),(-1,0),(1,\ell),(-1,\ell)\},
\label{eq:planar-higher-order}
\end{equation}
The target points are identical except that the two lower targets are exchanged, as shown in \cref{fig:higher-order}.

Taking $\boldsymbol{\alpha} = (1,-1,-1,1)$ gives $\sum \alpha_i \mathbf{p}_i = 0$ while $\sum \alpha_i \mathbf{q}_i = (-4, 0)$.
Therefore, \cref{lem:affine-relation} gives $\rho(\mathcal{C}) \ge 1$.

On the other hand, every triple can be aligned as in \cref{fig:higher-order-b}.
First, align the unique point in one row with its corresponding target.
The angle between its two sides is unchanged, so rotate or reflect the source to make the corresponding sides point in the same directions.
These sides have lengths $\ell$ and $\sqrt{\ell^2+4}$, which are exchanged in the target triangle.
The other two points therefore each have error $\sqrt{\ell^2+4}-\ell$, giving
\begin{equation}
\lim_{\ell\to\infty} \rho_3(\mathcal{C})\le \lim_{\ell\to\infty} \sqrt{\ell^2+4}-\ell = 0.
\label{eq:planar-triple-fit}
\end{equation}
\end{example}

We can extend \cref{ex:planar-higher-order} to every dimension.
First add zero coordinates to all points to embed them in higher dimensions.
In three dimensions, replace $\mathbf{p}_4$ by two matching points whose midpoint is the original $\mathbf{p}_4$.
In higher dimensions, use $d-1$ affinely independent points whose average is the original $\mathbf{p}_4$, with their spread in the additional directions increasing with $\ell$.
More formally, choose fixed vectors $\mathbf{v}_1,\ldots,\mathbf{v}_{d-1}$ spanning $\mathbb{R}^{d-2}$ and satisfying $\sum_{i=1}^{d-1}\mathbf{v}_i=\mathbf{0}$.
Replace the correspondence $\mathbf{p}_4$ by
\begin{equation}
\mathbf{p}_{i+3}=(-1,\ell,\ell\mathbf{v}_i),\quad i=1,\ldots,d-1.
\label{eq:general-obstruction}
\end{equation}
Set $\mathbf{q}_1=\mathbf{p}_2$ and $\mathbf{q}_2=\mathbf{p}_1$, keeping all remaining target points identical to their sources.
This construction gives the following result, which is proved in \cref{app:counterexamples}.
\begin{proposition}[Limits of group-$(d+1)$ tests]
\label{prop:lower-order-obstruction}
For every $d\ge2$, there are correspondence sets $\mathcal{C}_\ell$ in $\mathbb{R}^d$ such that
\begin{equation}
\begin{aligned}
\lim_{\ell\to\infty}\rho_{d+1}(\mathcal{C}_\ell)=0,\qquad
\rho(\mathcal{C}_\ell) \ge 1 \quad \forall \ell>0.
\end{aligned}
\label{eq:general-obstruction-bound}
\end{equation}
\end{proposition}
The size ratio $\Omega_k(\tau)/\mathrm{MC}(\tau)$ can also be arbitrarily large for $2\le k\le d$, at any fixed tolerance $\tau>0$.
An extension of the parabola construction in \cref{ex:unbounded-size-ratio} establishes this in \cref{app:group-size-ratio}.

\subsection{First sufficient group size}
\Cref{prop:lower-order-obstruction} shows that group-$(d+1)$ consistency cannot guarantee an alignment error bound.
Increasing the group size by one, however, gives a guarantee independent of the number of correspondences and their geometry.
\begin{theorem}[Group-$k$ consistency]
\label{thm:higher-order}
For every correspondence set $\mathcal{C}$ of points in $\mathbb R^d$,
\begin{equation}
\rho(\mathcal{C})\le2\sqrt d\,\rho_{d+2}(\mathcal{C}).
\label{eq:full-local-global}
\end{equation}
\end{theorem}
In 3D, if every subset of up to five correspondences can be aligned within $\tau$, possibly by different transformations, then one rigid transformation aligns the entire set within $2\sqrt3\,\tau$.

The remainder of this subsection develops the proof of \cref{thm:higher-order}.
We first relate the alignment error to weighted least-squares fitting, then show how local consistency bounds these weighted errors.
Proofs of the intermediate lemmas are given in \cref{app:alignment-bounds}.

For weights $\mathbf w$ with $\|\mathbf w\|_2=1$, define the optimal weighted least-squares error
\begin{equation}
E(\mathbf w)=\min_{T\in\mathrm E(d)}
\sum_{i\in\mathcal C}w_i^2
\|T(\mathbf p_i)-\mathbf q_i\|^2.
\label{eq:weighted-fit}
\end{equation}
This is a weighted orthogonal Procrustes problem, solvable using singular value decomposition~\cite{arun1987least}.
Maximizing $E(\mathbf w)$ over the weights gives the Lagrangian dual problem of $\rho(\mathcal C)^2$
\begin{equation}
\underline{\rho}(\mathcal C)^2
=\max_{\|\mathbf w\|_2=1}E(\mathbf w).
\label{eq:dual-alignment error}
\end{equation}

\begin{lemma}[Duality gap bound]
\label{lem:weighted-fitting}
For a correspondence set $\mathcal{C}$,
\begin{equation}
\underline{\rho}(\mathcal C)
\le \rho(\mathcal C)
\le
\sqrt{2}\underline{\rho}(\mathcal C).
\label{eq:weighted-fitting-bound}
\end{equation}
\end{lemma}

To bound $E(\mathbf w)$ using local consistency, we return to the zero-sum point combinations from \cref{lem:affine-relation}, such as vectors between midpoints.
The following lemma expresses the weighted fitting error as a sum of squared length differences between corresponding source and target combinations.
\begin{lemma}[Weighted error as length differences]
\label{lem:few-vector-representation}
For every $\mathbf w$ with $\|\mathbf w\|_2=1$, there exist orthonormal vectors $\mathbf v_1,\ldots,\mathbf v_r$, with $r\le2d$, such that, setting $\boldsymbol\alpha^{(j)}=\mathbf w\odot\mathbf v_j$,
\begin{align}
E(\mathbf w)&=\sum_{j=1}^r
\bigl(\|\textstyle\sum_i\alpha_i^{(j)}\mathbf p_i\|
-\|\textstyle\sum_i\alpha_i^{(j)}\mathbf q_i\|\bigr)^2,
\label{eq:few-vector-representation}\\
\mathbf v_j^\top\mathbf w&=0,\qquad j=1,\ldots,r.
\label{eq:few-vector-orthogonality}
\end{align}
\end{lemma}
Here $\odot$ denotes the element-wise product.
The orthogonality condition gives $\sum_i\alpha_i^{(j)}=\mathbf w^\top\mathbf v_j=0$,
so each term involves a zero-sum combination, considered in \cref{lem:affine-relation}.
Moreover, Cauchy--Schwarz gives $\sum_i|\alpha_i^{(j)}| \le\|\mathbf w\|_2\|\mathbf v_j\|_2=1$.
We can now bound these length differences using local consistency, extending \cref{lem:affine-relation}.
\begin{lemma}[Local zero-sum combination bound]
\label{lem:local-global-affine}
For a correspondence set $\mathcal{C}$ and coefficients satisfying ${\sum_{i\in\mathcal{C}}\alpha_i=0}$,
\begin{equation}
\bigl|
\|\textstyle\sum_i\alpha_i \mathbf{p}_i\|
-
\|\textstyle\sum_i\alpha_i \mathbf{q}_i\|
\bigr|
\le
\rho_{d+2}(\mathcal{C})\sum_i|\alpha_i|.
\label{eq:signed-comparison}
\end{equation}
\end{lemma}
Although the coefficients can involve the entire set, the bound depends on at most $d+2$ correspondences at a time.

\begin{proof}[Proof of \cref{thm:higher-order}]
\Cref{lem:local-global-affine} bounds each length difference in \cref{lem:few-vector-representation} by $\rho_{d+2}(\mathcal C)$.
With at most $2d$ squared terms, $E(\mathbf w)\le2d\rho_{d+2}(\mathcal C)^2$ for every $\mathbf w$.
Maximizing over $\mathbf w$ gives $\underline{\rho}(\mathcal C)^2\le2d\rho_{d+2}(\mathcal C)^2$, and \cref{lem:weighted-fitting} gives~\cref{eq:full-local-global}.
\end{proof}

\subsection{Proper rigid motions}
\label{ssec:proper-motion}

So far, allowing reflections has strengthened the counterexamples by allowing a larger transformation class.
For group-$k$ consistency, however, proper rigid motions behave differently.
Restricting orientation changes the first $k$ that guarantees a bounded alignment error.
We write $\rho^+(\mathcal{C})$ for~\cref{eq:rho} with the optimization restricted to $\mathrm{SE}(d)$, and define $\rho_k^+(\mathcal{C})$ analogously.

\begin{lemma}[Local orientation equivalence]
\label{lem:proper-local-equivalence}
For a correspondence set $\mathcal{I}$ with $|\mathcal{I}|\le d$, $\rho^+(\mathcal{I})=\rho(\mathcal{I})$.
Consequently, $\rho_k^+(\mathcal{C})=\rho_k(\mathcal{C})$ for every $k \le d$.
\end{lemma}
The proof is given in \cref{app:proper-motions}.
Together with \cref{prop:lower-order-obstruction}, this rules out an alignment error bound for $k\le d$.
By \cref{lem:proper-local-equivalence}, \cref{prop:lower-order-obstruction} gives
\begin{equation}
\lim_{\ell\to\infty}\rho_d^+(\mathcal{C}_\ell)=0,
\qquad
\rho^+(\mathcal{C}_\ell)\ge\rho(\mathcal{C}_\ell)\ge1
\quad\forall\ell>0.
\label{eq:proper-lower-order-obstruction}
\end{equation}

The unbounded size-ratio construction in \cref{app:group-size-ratio}, similar to \cref{ex:unbounded-size-ratio}, also applies to proper motions for $2\le k\le d$.

At $k=d+1$, the orientation restriction becomes useful.
Comparing two proper rigid motions on $d$ source anchors bounds their disagreement over the entire cloud, with coefficients determined by the source geometry.
\begin{lemma}[Anchor control]
\label{lem:anchor-control}
Let $\mathcal P=\{\mathbf p_1,\ldots,\mathbf p_n\}\subset\mathbb R^d$, with $d\ge2$, and let $\mathcal A$ be an index set for $d$ affinely independent source anchors.
For each $i$ and $j\in\mathcal A$, define
\begin{equation}
w_{ij}=
\frac{\operatorname{vol}\bigl(\{\mathbf p_k:k\in\mathcal A\setminus\{j\}\}\cup\{\mathbf p_i\}\bigr)}
{\operatorname{vol}\{\mathbf p_k : k\in\mathcal A\}},
\label{eq:anchor-geometric-weights}
\end{equation}
where $\operatorname{vol}$ denotes the $(d-1)$-dimensional volume of the convex hull of the listed points.
Then any $T,U\in\mathrm{SE}(d)$ satisfy
\begin{equation}
\|T(\mathbf p_i)-U(\mathbf p_i)\|
\le\sum_{j\in\mathcal A}w_{ij}\|T(\mathbf p_j)-U(\mathbf p_j)\|.
\label{eq:weighted-anchor-control}
\end{equation}
Consequently, with $\kappa=\max_i\sum_{j\in\mathcal A}w_{ij}$,
\begin{equation}
\max_i\|T(\mathbf p_i)-U(\mathbf p_i)\|
\le\kappa\max_{j\in\mathcal A}\|T(\mathbf p_j)-U(\mathbf p_j)\|.
\label{eq:anchor-control}
\end{equation}
\end{lemma}
The proof is given in \cref{app:proper-motions}.
Each $w_{ij}$ compares the anchor simplex with the one obtained by replacing $\mathbf p_j$ by $\mathbf p_i$.
Choosing a maximum-volume anchor simplex therefore ensures $w_{ij}\le1$ and $\kappa\le d$.
In 2D, a farthest pair, and in 3D, a maximum-area triangle suffices.
This choice gives a bound independent of the source geometry.
\begin{proposition}[Proper group-$k$ consistency]
\label{prop:proper-higher-order}
For every correspondence set $\mathcal{C}$ of points in $\mathbb{R}^d$, with $\tau=\rho_{d+1}^+(\mathcal{C})$,
\begin{equation}
\rho^+(\mathcal{C})
\le
(2d+1)\tau.
\label{eq:proper-local-global}
\end{equation}
\end{proposition}

\begin{proof}
For a full-dimensional source cloud, choose the indices $\mathcal A$ of a maximum-volume $(d-1)$-simplex.
\Cref{lem:anchor-control} then gives $\kappa\le d$.
If the source affine dimension is $r<d$, choose a maximum-volume $r$-simplex instead.
Its $r+1\le d$ anchors give the same weighted control by affine interpolation, with absolute barycentric coefficients $w_{ij}\le1$ and $\kappa\le r+1\le d$.

Let $T_0$ fit $\mathcal A$ within $\tau$, and let $T_i$ fit $\mathcal A\cup\{i\}$ within $\tau$.
Their disagreements on the anchors are at most $2\tau$, so
\begin{equation}
\|T_0(\mathbf p_i)-\mathbf q_i\|
\le\sum_{j\in\mathcal A}w_{ij}\|T_0(\mathbf p_j)-T_i(\mathbf p_j)\|+\tau
\le(2\kappa+1)\tau\le(2d+1)\tau.
\label{eq:proper-fit-bound}
\end{equation}
Thus $T_0$ is a common fit with the claimed error.
\end{proof}

\subsection{Sparse anchor certificates}
\label{sec:sparse-certificates}
\label{sec:weighted-anchor-certificates}
Anchor control also gives a way to check whether a selected correspondence set can be aligned within a prescribed tolerance.
We consider proper motions in $d\ge2$ and fix $d$ affinely independent source anchors from the selected set $\mathcal C$, with index set $\mathcal A$.
For the geometric weights in \cref{lem:anchor-control}, write $\kappa_i=\sum_{j\in\mathcal A}w_{ij}$.

Let $T_0\in\mathrm{SE}(d)$ be any reference transformation, with residuals $e_i=\|T_0(\mathbf p_i)-\mathbf q_i\|$.
Their maximum gives an upper bound on the alignment error.
They also give a lower bound: if another transformation fits $\mathcal A\cup\{i\}$ within $b$, it differs from $T_0$ by at most $e_j+b$ on anchor $j$.
\Cref{eq:weighted-anchor-control} then gives $e_i\le\sum_{j\in\mathcal A}w_{ij}(e_j+b)+b$, and hence
\begin{equation}
\widetilde\ell_i
=\frac{[e_i-\sum_{j\in\mathcal A}w_{ij}e_j]_+}{1+\kappa_i}
\le\rho^+(\mathcal A\cup\{i\}),
\qquad [x]_+=\max\{x,0\}.
\label{eq:anchor-cheap-lower}
\end{equation}
Combining the lower and upper bounds for the full set gives
\begin{equation}
\max_{i\in\mathcal C}\widetilde\ell_i
\le\rho^+(\mathcal C)
\le\max_{i\in\mathcal C}e_i.
\label{eq:anchor-fit-bracket}
\end{equation}
An upper bound at most $\tau$ certifies that the set can be aligned and a lower bound above $\tau$ certifies that it cannot.
Otherwise, the test is inconclusive.
With anchors fixed, these bounds require one reference transformation and a linear scan for fixed $d$.

A weighted least-squares fit gives additional information about how much another transformation can improve the residual at a point while still fitting the anchors.
Using weights $\mathbf w=(w_j)_{j\in\mathcal A}$ with $\|\mathbf w\|_2=1$, define the proper-motion counterpart of~\cref{eq:weighted-fit},
\begin{equation}
E^+(\mathbf w)=\min_{T\in\mathrm{SE}(d)}
\sum_{j\in\mathcal A}w_j^2\|T(\mathbf p_j)-\mathbf q_j\|^2.
\label{eq:weighted-anchor-loss}
\end{equation}
Let $T_{\mathbf w}$ attain the minimum in~\cref{eq:weighted-anchor-loss}, and write $e_i(\mathbf w)=\|T_{\mathbf w}(\mathbf p_i)-\mathbf q_i\|$.
If $E^+(\mathbf w)>\tau^2$, the anchors cannot be aligned within $\tau$.
Otherwise, every transformation fitting the anchors within $\tau$ can increase their weighted fitting error by at most $\tau^2-E^+(\mathbf w)$.
We use this remaining error to bound how far it can move each point from $T_{\mathbf w}$.

Define the weighted centroids
$\bar{\mathbf p}_{\mathbf w}=\sum_{j\in\mathcal A}w_j^2\mathbf p_j$ and
$\bar{\mathbf q}_{\mathbf w}=\sum_{j\in\mathcal A}w_j^2\mathbf q_j$.
Let $\mathbf R_{\mathbf w}$ be the rotation of $T_{\mathbf w}$, and set
\begin{equation}
\mathbf x_i=\mathbf p_i-\bar{\mathbf p}_{\mathbf w},
\qquad
\mathbf B_{\mathbf w}
=\mathbf R_{\mathbf w}^{\top}
\sum_{j\in\mathcal A}w_j^2
(\mathbf q_j-\bar{\mathbf q}_{\mathbf w})
(\mathbf p_j-\bar{\mathbf p}_{\mathbf w})^\top.
\label{eq:anchor-svd-geometry}
\end{equation}
The matrix $\mathbf B_{\mathbf w}$ is symmetric at the optimum and measures rotational cost in the source frame.
If its eigenvalues $b_1\le\cdots\le b_d$ satisfy
\begin{equation}
b_1+b_2>0,
\label{eq:proper-svd-curvature}
\end{equation}
then this cost is positive in every rotation plane.
Define
\begin{equation}
\chi_i(\mathbf w)^2
=1+
\max_{\substack{\boldsymbol\Omega^\top=-\boldsymbol\Omega\\
\boldsymbol\Omega\ne\mathbf0}}\;
\frac{\|\boldsymbol\Omega\mathbf x_i\|^2}
{\operatorname{tr}(\boldsymbol\Omega^\top\boldsymbol\Omega\mathbf B_{\mathbf w})}.
\label{eq:anchor-svd-sensitivity}
\end{equation}
The quotient compares rotational displacement at $\mathbf p_i$ with its cost on the anchors and additional $1$ accounts for translation.

\begin{proposition}[Weighted anchor certificate]
\label{prop:weighted-anchor-svd}
Suppose~\cref{eq:proper-svd-curvature} holds and $E^+(\mathbf w)\le\tau^2$.
Then every $T\in\mathrm{SE}(d)$ fitting the anchors within $\tau$ satisfies
\begin{equation}
\|T(\mathbf p_i)-T_{\mathbf w}(\mathbf p_i)\|
\le\chi_i(\mathbf w)\sqrt{\tau^2-E^+(\mathbf w)}.
\label{eq:weighted-svd-displacement}
\end{equation}
Consequently, $\mathcal A\cup\{i\}$ cannot be aligned within $\tau$ whenever
\begin{equation}
\tau+\chi_i(\mathbf w)\sqrt{\tau^2-E^+(\mathbf w)} < e_i(\mathbf w).
\label{eq:anchor-svd-rejection}
\end{equation}
\end{proposition}
The proof is given in \cref{app:svd-displacement}.

Each test identifies a subset $\mathcal A\cup\{i\}$ that cannot be aligned within $\tau$.
With anchors and dimension fixed, testing all points takes linear time.
\Cref{app:weighted-anchor-svd} gives formulas for computing $\chi_i(\mathbf w)$ and extends the certificate to $\mathrm E(d)$.

\section{Computational Complexity}
\label{sec:complexity}
\begin{figure}[tp]
\centering
\subfloat[$H$\label{fig:co2-a}]{%
\begin{minipage}[c][1.9in][c]{.105\linewidth}
\centering
\includegraphics[width=\linewidth]{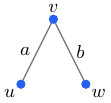}
\end{minipage}}\hfill
\subfloat[$S_2(H)$\label{fig:co2-b}]{%
\begin{minipage}[c][1.9in][c]{.16\linewidth}
\centering
\includegraphics[width=\linewidth]{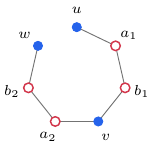}
\end{minipage}}\hfill
\subfloat[$F(H)$\label{fig:co2-c}]{%
\begin{minipage}[c][1.9in][c]{.16\linewidth}
\centering
\includegraphics[width=\linewidth]{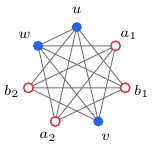}
\end{minipage}}\hfill
\subfloat[Correspondences\label{fig:co2-d}]{%
\begin{minipage}[c][1.9in][c]{.38\linewidth}
\centering
\includegraphics[width=\linewidth]{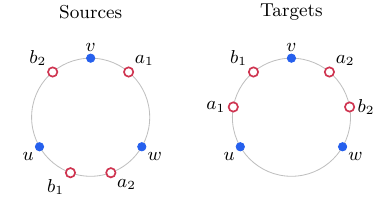}
\end{minipage}}\hfill
\subfloat[One edge\label{fig:co2-e}]{%
\begin{minipage}[c][1.9in][c]{.16\linewidth}
\centering
\includegraphics[width=\linewidth,height=1.9in,keepaspectratio]{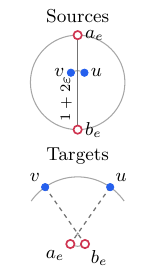}
\end{minipage}}\par
\caption{
The co-2-subdivision of any graph $H$ is a consistency graph.
(a) A graph $H$ with edges $a=uv$ and $b=vw$.
(b) Its 2-subdivision $S_2(H)$, inserting $a_1,b_1$ on edge $a$ and $a_2,b_2$ on edge $b$.
(c) The complement $F(H)$.
(d) Seven planar correspondences whose consistency graph is $F(H)$ at edge tolerance $1$.
(e) One edge $e=uv$ in the construction for \cref{prop:clique-hardness}: the source diameter makes $a_e,b_e$ inconsistent, while opposite target rays make $u,a_e$ and $v,b_e$ inconsistent.
}
\label{fig:co2}
\end{figure}

For completeness, we also compare the computational complexity of the two formulations.
Maximum clique is NP-hard on arbitrary graphs~\cite{karp1972reducibility}, but this alone does not establish hardness on geometrically restricted consistency graphs.
For example, the pairwise translation-consistency graphs in \cref{sec:translation} are intersection graphs of equal-radius balls.
Maximum clique on these graphs is polynomial-time solvable for $d=1,2$~\cite{clark1990unitdisk} and NP-hard for $d\ge4$, while its exact complexity for $d=3$ remains open~\cite{bonamy2021eptas}.

Similarly, linear maximum consensus is NP-hard when its parameter dimension is part of the input~\cite{chin2018robust}, which does not establish hardness for rigid registration in a fixed dimension.
We therefore compare maximum clique on registration consistency graphs with maximum consensus over $\mathrm E(d)$ and $\mathrm{SE}(d)$ in each fixed dimension.

\begin{samepage}
\begin{proposition}[Maximum clique hardness]
\label{prop:clique-hardness}
At edge tolerance $1$, maximum clique on correspondence-consistency graphs is NP-hard in every fixed dimension $d\ge2$.
\end{proposition}
\end{samepage}

\begin{proof}[Proof sketch]
We realize the co-2-subdivision $F(H)$ of any graph $H$ as a planar consistency graph at edge tolerance $1$.
Maximum clique is NP-hard on this graph class~\cite{bonamy2021eptas}.
Starting with $H$ in \cref{fig:co2-a}, replace each edge $uv$ by the path $u-a_e-b_e-v$ to obtain $S_2(H)$ in \cref{fig:co2-b}, then take its complement $F(H)$ in \cref{fig:co2-c}.
The resulting graph satisfies
\begin{equation}
\omega(F(H))=|E(H)|+\alpha(H),
\label{eq:subdivision}
\end{equation}
where $\omega$ is the clique number and $\alpha(H)$ is the largest independent-set size in $H$.
\Cref{fig:co2-d} shows a planar realization of this example at edge tolerance $1$.

For this realization, choose rational unit directions $\mathbf c_w$ for vertices and $\mathbf c_e$ for edges on a short arc, distinct within each list.
Number the $m=|E(H)|$ edges $1,\ldots,m$.
For a sufficiently small rational $\varepsilon>0$, set $s_e=\varepsilon+e\varepsilon^2/(m+1)$ and use the planar correspondences
\begin{equation}
\begin{array}{c|cc}
\text{vertex}&\mathbf p&\mathbf q\\\hline
w&\varepsilon^2\mathbf c_w&(\tfrac32+2\varepsilon^2)\mathbf c_w\\
a_e&(\tfrac12+\varepsilon)\mathbf c_e&-s_e\mathbf c_u\\
b_e&-(\tfrac12+\varepsilon)\mathbf c_e&-s_e\mathbf c_v.
\end{array}
\label{eq:hardness-coordinates}
\end{equation}
Opposite source endpoints make $a_e,b_e$ inconsistent, while the target rays make each new vertex inconsistent with its assigned original endpoint; see \cref{fig:co2-e}.
The angular separations and radial offsets keep every other pair consistent at unit tolerance.
\Cref{app:hardness} gives the choice of $\varepsilon$ and verifies these inequalities, with distinct coordinates of polynomial bit length.
Thus the consistency graph is exactly $F(H)$.
Appending zero coordinates preserves all distances and gives the result for every $d\ge2$.
\end{proof}

\begin{samepage}
\begin{proposition}[Maximum consensus complexity]
\label{prop:consensus-complexity}
For each fixed $d\ge1$, the maximum-consensus value and an optimal subset can be computed in polynomial time, for both $\mathrm E(d)$ and $\mathrm{SE}(d)$.
\end{proposition}
\end{samepage}

\begin{proof}
Treat the entries of $(\mathbf R,\mathbf t)$ as real variables, with $\mathbf R^\top\mathbf R=\mathbf I$ and, for proper motions, $\det\mathbf R=1$.
Each inlier test is the polynomial inequality
\begin{equation}
f_i(\mathbf R,\mathbf t)
=\|\mathbf R\mathbf p_i+\mathbf t-\mathbf q_i\|^2-\tau^2\le0.
\label{eq:rigid-consensus-inlier}
\end{equation}
There are $d^2+d$ variables, $n$ quadratic residual tests, and $d(d+1)/2$ quadratic orthogonality equations; proper motions add one determinant equation of degree $d$.
For fixed $d$, the variable count and maximum degree are therefore fixed, so all realizable sign patterns can be enumerated in polynomial time~\cite{renegar1992}.
Geometrically, these patterns describe regions with constant inlier sets, including their boundary cases.
A pattern satisfying the motion equations with the most nonpositive residuals therefore identifies an exact maximum consensus set.
\end{proof}

This is a worst-case comparison of the formulations and does not predict the runtimes of practical solvers.
\Cref{app:hardness} gives the full hardness proof.

\section{Experimental Evaluations}
\label{sec:practical-evidence}
In this section, we evaluate the consistency formulations and alignment bounds on real point cloud registration data.
Specifically, we compare the runtimes of pairwise and group-$k$ selection, measure the alignment errors of selected maximum cliques, and evaluate the proposed anchor certificates.
We also examine the relationship between joint feasibility of the selected matches and pose accuracy.

We use the indoor RGB-D benchmarks 3DMatch~\cite{zeng20173dmatch} and 3DLoMatch~\cite{huang2021predator}, the urban LiDAR benchmarks KITTI-10m and KITTI-LC~\cite{Geiger2012CVPR,qiao2024g3reg}, and the ETH LiDAR benchmark~\cite{pomerleau2012challenging}.
\Cref{tab:eval-datasets} summarizes the five datasets and the average number $n$ of tentative correspondences per scan pair.
We set the geometric thresholds in terms of the voxel size $v$, as listed in \cref{tab:eval-parameters}.

\begin{table}[!ht]
\begin{minipage}[t]{.56\linewidth}
\centering
\caption{Datasets and pose-success thresholds.}
\label{tab:eval-datasets}
\footnotesize
\setlength{\tabcolsep}{3pt}
\begin{tabular}{@{}lrrrrr@{}}
\toprule
Dataset & Pairs & Avg.\ $n$ & $v$ (cm) & RE (${}^\circ$) & TE (m) \\
\midrule
3DMatch   & 1,623 &   932 &  5 & 15 & 0.3 \\
3DLoMatch & 1,781 &   772 &  5 & 15 & 0.3 \\
KITTI-10m &   556 & 3,686 & 30 &  5 & 0.6 \\
KITTI-LC  & 3,325 & 3,808 & 30 &  5 & 2.0 \\
ETH       &   713 & 4,986 & 10 & 15 & 0.3 \\
\bottomrule
\end{tabular}
\end{minipage}\hfill
\begin{minipage}[t]{.41\linewidth}
\centering
\caption{Shared parameters.}
\label{tab:eval-parameters}
\footnotesize
\setlength{\tabcolsep}{3pt}
\begin{tabular}{@{}llr@{}}
\toprule
Setting & Symbol & Value \\
\midrule
Normals & $R_{\mathrm{normal}}, N_{\mathrm{normal}}$ & $2v$, 30 \\
FPFH & $R_{\mathrm{FPFH}}, N_{\mathrm{FPFH}}$ & $5v$, 100 \\
Distance tolerance & $\epsilon$ & $2v$ \\
Point tolerance & $\tau$ & $v$ \\
\bottomrule
\end{tabular}
\end{minipage}
\end{table}

We form tentative correspondences by mutual nearest neighbor matching of FPFH descriptors~\cite{rusu2009fast}.
To compute these descriptors, we first estimate each point's normal using at most $N_{\mathrm{normal}}$ nearest neighbors within radius $R_{\mathrm{normal}}$.
Using these normals, we compute FPFH descriptors with at most $N_{\mathrm{FPFH}}$ nearest neighbors within radius $R_{\mathrm{FPFH}}$.
For each scan pair, PMC~\cite{rossi2013pmc} selects a maximum clique $\mathcal C$ of the distance-consistency graph at edge tolerance $\epsilon$.
We estimate a proper rigid motion from $\mathcal C$ by least-squares fitting using SVD~\cite{arun1987least} and use this motion to report pose accuracy.

We evaluate whether the selected matches admit one proper rigid alignment within $\tau$, equivalently whether $\rho^+(\mathcal C)\le\tau$.
We compute verified lower and upper bounds on the minimax errors; a lower bound above $\tau$ proves infeasibility.
We also report pose success, which requires both rotation error (RE) and translation error (TE) to satisfy the thresholds in \cref{tab:eval-datasets}.

\subsection{Runtime comparison}
\begin{figure}[tb]
\centering
\begingroup
\footnotesize
\definecolor{RuntimeKTwo}{HTML}{CC79A7}
\definecolor{RuntimeKThree}{HTML}{0072B2}
\definecolor{RuntimeKFour}{HTML}{D55E00}
\definecolor{RuntimeKFive}{HTML}{009E73}
\tikz[baseline=-.5ex]{%
  \draw[black,line width=1.2pt,dash pattern=on 3pt off 2pt] (0,0)--(.5,0);
  \fill[black] (.25,0) circle[radius=1.2pt];}%
\,PMC\qquad
\tikz[baseline=-.5ex]{\draw[black,line width=1.1pt] (0,0)--(.5,0);}%
\,GkCM\qquad
\textcolor{RuntimeKTwo}{\rule[.5ex]{1.4em}{.9pt}}\,$k=2$\quad
\textcolor{RuntimeKThree}{\rule[.5ex]{1.4em}{.9pt}}\,$k=3$\quad
\textcolor{RuntimeKFour}{\rule[.5ex]{1.4em}{.9pt}}\,$k=4$\quad
\textcolor{RuntimeKFive}{\rule[.5ex]{1.4em}{.9pt}}\,$k=5$\par
\endgroup
\smallskip
\subfloat[Total\label{fig:eval-runtime-total}]{%
\begin{minipage}[b]{.35\linewidth}
\includegraphics[width=\linewidth]{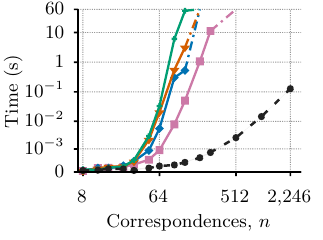}
\end{minipage}}\hfill
\subfloat[Clique search\label{fig:eval-runtime-search}]{%
\begin{minipage}[b]{.315\linewidth}
\includegraphics[width=\linewidth]{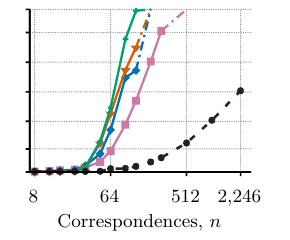}
\end{minipage}}\hfill
\subfloat[(Hyper)graph construction\label{fig:eval-runtime-construction}]{%
\begin{minipage}[b]{.315\linewidth}
\includegraphics[width=\linewidth]{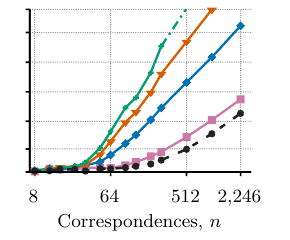}
\end{minipage}}\par
\caption{Single-thread runtime of exact GkCM and PMC on one 3DMatch scan pair.}
\label{fig:eval-runtime}
\end{figure}

\label{sec:consistency-runtime}
We compare PMC~\cite{rossi2013pmc} with the exact GkCM solver~\cite{forsgren2024gkcm}
on one 3DMatch scan pair, with different numbers of correspondences up to 2,246.
Each run uses one CPU core, with separate 60-second limits for each construction stage and for clique search.
\Cref{fig:eval-runtime} reports total runtime, clique search time, and graph or hypergraph construction time.

PMC solves the full pairwise graph in about $0.11$ seconds, while exact GkCM
reaches the 60-second limit at smaller inputs.
For $k=3$, hypergraph construction alone takes about 17 seconds for roughly 2,200 matches.
And for $k=4$, it exceeds 60 seconds at the same input size.
Among the tested configurations, only pairwise consistency with PMC completes across the full input range.
The remaining experiments use the maximum cliques selected by PMC.

\subsection{Limitations of pairwise consistency}
\begin{figure}[tb]
\centering
\begingroup
\footnotesize
\definecolor{AlignmentMatch}{HTML}{0072B2}
\definecolor{AlignmentLoMatch}{HTML}{D55E00}
\definecolor{AlignmentTenM}{HTML}{9B6D00}
\definecolor{AlignmentLC}{HTML}{009E73}
\definecolor{AlignmentETH}{HTML}{9A538D}
\tikz[baseline=-.5ex]{\draw[AlignmentMatch,line width=1.1pt] (0,0)--(.5,0);}%
\,3DMatch\qquad
\tikz[baseline=-.5ex]{\draw[AlignmentLoMatch,line width=1.1pt,dash pattern=on 3pt off 2pt] (0,0)--(.5,0);}%
\,3DLoMatch\qquad
\tikz[baseline=-.5ex]{\draw[AlignmentTenM,line width=1.1pt,dash pattern=on 5pt off 1pt on 1pt off 1pt] (0,0)--(.5,0);}%
\,KITTI-10m\qquad
\tikz[baseline=-.5ex]{\draw[AlignmentLC,line width=1.1pt,dash pattern=on 4pt off 2pt on 1pt off 2pt] (0,0)--(.5,0);}%
\,KITTI-LC\qquad
\tikz[baseline=-.5ex]{\draw[AlignmentETH,line width=1.1pt,dash pattern=on 1pt off 2pt] (0,0)--(.5,0);}%
\,ETH\par
\smallskip
\endgroup
\subfloat[$\rho^+(\mathcal C)/\tau$\label{fig:eval-alignment-a}]{%
\begin{minipage}[b]{.35\linewidth}
\includegraphics[width=\linewidth]{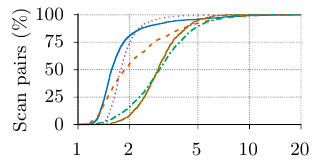}
\end{minipage}}\hfill
\subfloat[$\rho^+(\mathcal C)/\rho_3^+(\mathcal C)$\label{fig:eval-alignment-b}]{%
\begin{minipage}[b]{.315\linewidth}
\includegraphics[width=\linewidth]{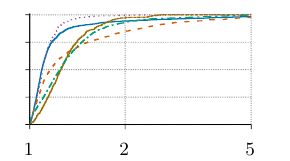}
\end{minipage}}\hfill
\subfloat[$\rho^+(\mathcal C)/\rho_4^+(\mathcal C)$\label{fig:eval-alignment-c}]{%
\begin{minipage}[b]{.315\linewidth}
\includegraphics[width=\linewidth]{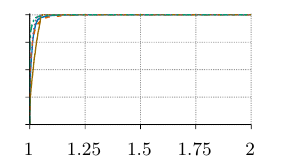}
\end{minipage}}\par
\smallskip
\subfloat[$\rho(\mathcal C)/\tau$\label{fig:eval-alignment-d}]{%
\begin{minipage}[b]{.35\linewidth}
\includegraphics[width=\linewidth]{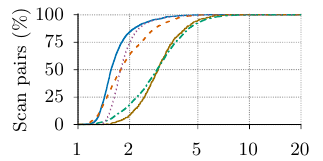}
\end{minipage}}\hfill
\subfloat[$\rho(\mathcal C)/\rho_3(\mathcal C)$\label{fig:eval-alignment-e}]{%
\begin{minipage}[b]{.315\linewidth}
\includegraphics[width=\linewidth]{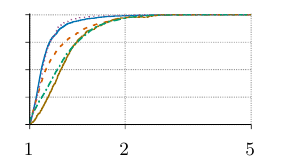}
\end{minipage}}\hfill
\subfloat[$\rho(\mathcal C)/\rho_4(\mathcal C)$\label{fig:eval-alignment-f}]{%
\begin{minipage}[b]{.315\linewidth}
\includegraphics[width=\linewidth]{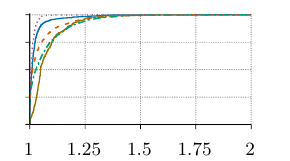}
\end{minipage}}\par
\caption{Cumulative distributions of maximum-clique alignment error and its ratios to group-consistency errors across five datasets.
The top row considers proper rigid motions, while the bottom row also allows reflections.
The left column shows the alignment error normalized by $\tau$; the middle and right columns show its ratios to the worst group-$3$ and group-$4$ alignment errors, respectively.
}
\label{fig:eval-alignment}
\end{figure}

We evaluate whether the selected maximum cliques admit a common alignment at point tolerance $\tau$, using the matched edge tolerance $\epsilon=2\tau$.
\Cref{fig:eval-alignment-a,fig:eval-alignment-d} show the cumulative distributions of $\rho^+(\mathcal C)/\tau$ and $\rho(\mathcal C)/\tau$, respectively.

For at least 99.8\% of scan pairs in each dataset, the selected maximum clique is certified infeasible at $\tau$, both with and without reflections.
More than 57\% of the selected sets have a proper-motion alignment error exceeding $2\tau$.
The median alignment error is at least $1.57\tau$ for 3DMatch, $1.87\tau$ for 3DLoMatch, $1.78\tau$ for ETH, and nearly $3\tau$ for both KITTI benchmarks.
The corresponding results for $E(3)$ show that these infeasibilities cannot be explained solely by restricting to proper motions.
\Cref{fig:eval-alignment-b,fig:eval-alignment-c} show the ratios of the whole-set proper-motion error $\rho^+(\mathcal C)$ to the group-$3$ and group-$4$ errors, respectively.
For group-$4$, the verified upper bound on $\rho^+(\mathcal C)/\rho_4^+(\mathcal C)$ is below $1.25$ for every pair across all five datasets.
Thus, the group-$4$ error is at least 80\% of the whole-set alignment error for every selected clique.
The theoretical factors of $7$ and $2\sqrt3$ in \cref{prop:proper-higher-order,thm:higher-order} are therefore conservative on these data.

Overall, \Cref{fig:eval-alignment} illustrates that higher-order consistency tests capture much of the whole-set alignment error, even though cliques satisfying pairwise consistency can exhibit large alignment errors across different datasets.

\subsection{Evaluation of anchor certificates}
We evaluate the certificates in \cref{sec:sparse-certificates} for detecting infeasible maximum cliques and rejecting candidate correspondences that cannot be aligned with fixed anchors.
All tests consider proper rigid motions in 3D and use one source anchor triangle per clique, selected as described in \cref{app:anchor-selection}.

The \emph{simple anchor bound} uses the largest anchor residual to obtain a lower bound on the alignment error through \cref{eq:anchor-fit-lower}.
\emph{Geometric weights} refine this bound using individual anchor residuals and source geometry, as in \cref{eq:anchor-cheap-lower}.
\emph{Uniform SVD} applies \cref{prop:weighted-anchor-svd} with equal fitting weights to bound how much a motion fitting the anchors can reduce each point's residual.
\emph{Weighted SVD} applies the same proposition with point-dependent fitting weights from \cref{eq:geometry-fitting-weights}, as described in \cref{app:svd-fitting-weights}.

\Cref{tab:eval-certificates} reports registration success, selected-clique infeasibility, and certificate rejection rates.
The exact feasibility solver provides reference infeasibility rates, while the certificate rows report the proportion of tested cliques or candidates rejected by each method.
For direct feasibility checks, we use SCIP~\cite{hojny2025scip}, imposing the proper rigid-motion constraints and pointwise residual inequalities in \cref{eq:rigid-consensus-inlier}.

\begin{table}[tb]
\centering

\caption{Registration success and infeasibility rates (\%), with mean evaluation time per scan pair (ms).
Clique rejection rates are relative to all selected maximum cliques, while candidate rejection rates are relative to correspondences passing the pairwise tests.
The exact feasibility solver provides reference infeasibility rates.}
\label{tab:eval-certificates}
\small
\setlength{\tabcolsep}{5pt}
\setlength{\aboverulesep}{0pt}
\setlength{\belowrulesep}{0pt}
\renewcommand{\arraystretch}{1.15}
\begin{tabular}{@{}lrrrrr|r@{}}
\toprule
Method & 3DMatch & 3DLoMatch & KITTI-10m & KITTI-LC & ETH & \shortstack{Time (ms)} \\
\midrule
\multicolumn{6}{@{}l|}{\textit{Pose and feasibility}} & \\
Registration success & 86.1 & 47.1 & 98.9 & 70.1 & 94.1 & --- \\
Exact feasibility solver & 99.9 & 99.8 & 100.0 & 100.0 & 100.0 & 389 \\
\midrule
\multicolumn{6}{@{}l|}{\textit{Feasibility certificates}} & \\
Simple anchor bound & 5.9 & 15.9 & 27.0 & 30.0 & 2.7 & 0.11 \\
Geometric weights & 30.9 & 50.8 & 95.0 & 87.8 & 48.5 & 0.06 \\
Uniform SVD & 67.7 & 83.0 & 98.7 & 97.1 & 83.9 & 0.08 \\
Weighted SVD & \textbf{78.9} & \textbf{88.0} & \textbf{99.3} & \textbf{98.5} & \textbf{92.4} & 0.18 \\
\midrule
\multicolumn{6}{@{}l|}{\textit{Anchor-based pruning}} & \\
Exact feasibility solver & 17.3 & 30.5 & 36.9 & 33.1 & 27.6 & 25,689 \\
Geometric weights & 4.4 & 11.3 & 18.0 & 16.1 & 9.7 & 0.16 \\
Uniform SVD & 10.3 & 21.5 & 28.2 & 25.6 & 19.6 & 0.34 \\
Weighted SVD & \textbf{12.3} & \textbf{23.7} & \textbf{32.1} & \textbf{28.9} & \textbf{22.5} & 1.10 \\
\bottomrule
\end{tabular}
\end{table}
The first block compares the pose accuracy of PMC followed by least-squares fitting with the feasibility of its selected cliques.
Registration success ranges from 47.1\% to 98.9\%, whereas 99.8--100.0\% of the selected cliques are infeasible at tolerance $\tau$.
Thus, a fitted transformation can satisfy the pose-accuracy thresholds in \cref{tab:eval-datasets} even when the selected correspondences do not admit a common alignment within $\tau$.

The second block evaluates the rejection of entire selected cliques.
Each certificate uses anchor-based necessary conditions to detect infeasibility without fitting the entire clique.
Geometric weights improve detection over the simple anchor bound on all five datasets, while accounting for the anchor-fitting error further increases rejection rates.
Weighted SVD achieves the highest detection rate on every dataset, rejecting 78.9--99.3\% of the selected cliques.
On 3DMatch, it increases detection from 67.7\% with uniform SVD to 78.9\%.
Its mean evaluation time is 0.18\,ms per selected clique, compared with 389\,ms for the exact feasibility solver.

The third block evaluates candidate rejection when the chosen anchors are retained.
Starting from a feasible anchor triangle, we test candidates from the original correspondence set that pass all pairwise consistency tests against the anchors.
The exact feasibility solver identifies 17.3--36.9\% of these candidates as incompatible with the anchors at tolerance $\tau$.
Weighted SVD rejects 12.3--32.1\% of the tested candidates, compared with 4.4--18.0\% for geometric weights.
On 3DMatch, weighted SVD rejects 12.3\% of candidates against a reference infeasibility rate of 17.3\%; on KITTI-10m, the corresponding rates are 32.1\% and 36.9\%.

\section{Conclusion}
We have shown that maximum clique can differ arbitrarily from maximum consensus in both selected-set size and alignment error.
For fixed rotation, classical geometry provides a bound on alignment error and establishes the equivalence between maximum consensus and selecting the largest group-$k$ consistent set at $k=d+1$.
For full registration, our squared-distance bound accounts for the scale of each pair, while group-$k$ tests can remove dependence on point cloud scale and geometry altogether.
The smallest sufficient group sizes are $d+1$ for proper rigid motions and $d+2$ when reflections are allowed.
These results establish theoretical guarantees and limitations of pairwise and subset consistency, independent of any particular registration pipeline.

\section*{Data and code}
The code and datasets used to generate the results in \cref{sec:practical-evidence}, together with interactive examples and Lean~\cite{moura2021lean} formalizations of the proofs, are available at \url{https://ariarobotics.github.io/mc-vs-mc-registration/}.

\section*{Acknowledgments}
The authors used a large language model (OpenAI, GPT-5.6) for minor language editing. All technical content and results were developed and verified by the authors, who take full responsibility.

\bibliographystyle{siamplain}
\bibliography{references}

\clearpage
\renewcommand{\siampretitle}{Supplementary Materials: }
\renewcommand{\siamprelabel}{SM}
\renewcommand{\thepage}{\arabic{page}}
\setcounter{section}{0}
\setcounter{equation}{0}
\setcounter{theorem}{0}
\setcounter{figure}{0}
\setcounter{table}{0}
\setcounter{algorithm}{0}
\headers{Maximum Clique vs. Maximum Consensus Registration}{A. Babanazari and K. Fathian}
\pdfbookmark[0]{Supplementary Materials}{supplement}
\arxivsupplementtitle{On Maximum Clique and Maximum Consensus Formulations in Point Cloud Registration}

These supplementary materials provide additional proofs and constructions for the main article.
All notation and references to its results follow the main article.

\section{Proofs of the geometric results}
\label{app:geometric-proofs}
We give the counterexample arguments, the weighted-fitting proofs, and the proper-motion anchor argument in turn.
\Cref{app:sparse-certificates,app:weighted-anchor-svd} develop the geometric and weighted least-squares certificates; \cref{app:hardness} proves maximum clique hardness.

\subsection{Zero-sum relations and the lower-order obstruction}
\label{app:counterexamples}
\begin{proof}[Proof of \cref{lem:affine-relation}]
Choose a transformation with residuals $\mathbf e_i=\mathbf R\mathbf p_i+\mathbf t-\mathbf q_i$ satisfying $\|\mathbf e_i\|\le\rho(\mathcal{C})$.
Since $\sum_i\alpha_i=0$, we have $\mathbf{R}\sum_i\alpha_i\mathbf{p}_i-\sum_i\alpha_i\mathbf{q}_i=\sum_i\alpha_i\mathbf{e}_i$.
The triangle inequality gives~\cref{eq:affine-relation}, since $\mathbf R$ preserves lengths.
\end{proof}

\begin{proof}[Proof of \cref{prop:lower-order-obstruction}]
Let the centroids of the replacement points be
\begin{equation}
\bar{\mathbf p}_4=\frac{1}{d-1}\sum_{r=1}^{d-1}\mathbf p_{r+3}
=(-1,\ell,\mathbf0),
\qquad
\bar{\mathbf q}_4=\bar{\mathbf p}_4.
\label{eq:appendix-obstruction-centroid}
\end{equation}
The planar zero-sum relation becomes
\begin{equation}
\mathbf p_1-\mathbf p_2-\mathbf p_3+\bar{\mathbf p}_4=\mathbf0,
\qquad
\mathbf q_1-\mathbf q_2-\mathbf q_3+\bar{\mathbf q}_4=-4\mathbf e_1.
\label{eq:appendix-obstruction-relation}
\end{equation}
After expanding the centroid, the coefficients have total absolute value $4$, so \cref{lem:affine-relation} gives $\rho(\mathcal C_\ell)\ge1$.

It remains to fit each $(d+1)$-element subset $\mathcal I$; smaller subsets inherit these fits.
If $\mathcal I$ omits index $1$ or $2$, apply the fit from \cref{ex:planar-higher-order} in the first two coordinates and leave the others unchanged.
All replacement points share the same planar coordinates, so
\begin{equation}
\rho(\mathcal I)\le\sqrt{\ell^2+4}-\ell\longrightarrow0.
\label{eq:appendix-obstruction-planar-fit}
\end{equation}

Otherwise, $\mathcal I$ contains both swapped correspondences and $d-1$ unchanged points.
Let $H_\ell$ be the hyperplane through the origin and these unchanged points.
Their last $d-1$ coordinates, divided by $\ell$, are the vectors $(1,\mathbf v_r)$, possibly with one replaced by their average $(1,\mathbf0)$.
The spanning and zero-sum conditions make $(1,\mathbf v_r)$ a basis, and replacing one vector by their average preserves independence.
Since the first coordinates remain bounded, a unit normal $\mathbf n_\ell$ to $H_\ell$ can be chosen with $\mathbf n_\ell\to\mathbf e_1$.
Reflection in $H_\ell$ fixes every unchanged point and satisfies
\begin{equation}
\mathbf S_\ell=\mathbf I-2\mathbf n_\ell\mathbf n_\ell^\top,
\qquad
\mathbf S_\ell\mathbf p_1\longrightarrow\mathbf q_1,
\qquad
\mathbf S_\ell\mathbf p_2\longrightarrow\mathbf q_2.
\label{eq:appendix-obstruction-reflection}
\end{equation}
Thus $\rho(\mathcal I)\to0$ in both cases.
There are finitely many such subsets, so $\rho_{d+1}(\mathcal C_\ell)\to0$.
\end{proof}

\subsection{Unbounded size ratios for smaller groups}
\label{app:group-size-ratio}
\begin{example}[Higher-order extension of the parabola construction]
\label{ex:higher-order-size-ratio}
Fix $2\le k\le d$, $\tau>0$, and any $n\ge k+1$.
For $i=1,\ldots,n$, put
\begin{equation}
\mathbf u_i=(i,i^2,\ldots,i^{k-1}),
\qquad
\mathbf p_i=(\ell\mathbf u_i,0,\mathbf0_{d-k}),
\qquad
\mathbf q_i=\mathbf p_i+a i^k\mathbf e_k,
\label{eq:polynomial-size-construction}
\end{equation}
where $a=2^{k+1}\tau/k!$ and $\ell$ will be sufficiently large.
For $k=2$, this is the parabola construction in \cref{ex:unbounded-size-ratio}, with shifted sample indices.

For any $k$-element subset $\mathcal I$, polynomial interpolation gives a scalar $b_{\mathcal I}$ and a vector $\mathbf c_{\mathcal I}$ such that
\begin{equation}
i^k=b_{\mathcal I}+\mathbf c_{\mathcal I}^\top\mathbf u_i
\qquad(i\in\mathcal I).
\label{eq:polynomial-local-interpolation}
\end{equation}
Tilt the source in the plane spanned by $(\mathbf c_{\mathcal I},0,\mathbf0_{d-k})$ and $\mathbf e_k$ through angle $\arctan(a\|\mathbf c_{\mathcal I}\|/\ell)$ toward $\mathbf e_k$, and translate by $a b_{\mathcal I}\mathbf e_k$.
This proper rigid motion satisfies
\begin{equation}
\|\mathbf R_{\mathcal I}\mathbf p_i+a b_{\mathcal I}\mathbf e_k-\mathbf q_i\|
\le\frac{a^2\|\mathbf c_{\mathcal I}\|^2\|\mathbf u_i\|}{2\ell}
\longrightarrow0
\qquad(i\in\mathcal I).
\label{eq:polynomial-local-fit}
\end{equation}
There are finitely many such subsets, so sufficiently large $\ell$ makes the full set group-$k$ consistent at tolerance $\tau$.

In contrast, consider any $k+1$ indices $y_0<\cdots<y_k$ and set
\begin{equation}
\alpha_j=\frac{1}{\prod_{r\ne j}(y_j-y_r)},
\qquad j=0,\ldots,k.
\label{eq:polynomial-divided-differences}
\end{equation}
Lagrange interpolation gives the first two identities below, while integer spacing gives the last bound:
\begin{equation}
\sum_{j=0}^k\alpha_j y_j^r=0\quad(0\le r<k),
\qquad
\sum_{j=0}^k\alpha_j y_j^k=1,
\qquad
\sum_{j=0}^k|\alpha_j|\le\sum_{j=0}^k\frac{1}{j!(k-j)!}=\frac{2^k}{k!}.
\label{eq:polynomial-affine-relation}
\end{equation}
Consequently, $\sum_j\alpha_j\mathbf p_{y_j}=\mathbf0$ and $\sum_j\alpha_j\mathbf q_{y_j}=a\mathbf e_k$.
By \cref{lem:affine-relation}, every such subset has alignment error at least $ak!/2^k=2\tau$.
Every $k$-element subset is alignable, but no $k+1$ matches are, so
\begin{equation}
\Omega_k(\tau)=n,
\qquad
\mathrm{MC}(\tau)=k,
\qquad
\frac{\Omega_k(\tau)}{\mathrm{MC}(\tau)}=\frac{n}{k}\longrightarrow\infty.
\label{eq:group-clique-consensus-ratio}
\end{equation}
The local fits are proper, and the lower bound permits reflections, so the same unbounded size ratio holds when both selection problems are restricted to proper motions.
\end{example}

\subsection{Weighted fitting and alignment bounds}
\label{app:alignment-bounds}
\begin{proof}[Proof of \cref{lem:weighted-fitting}]
For any transformation, the weighted mean squared error is at most its largest squared residual, giving $\underline\rho\le\rho$.
Put $b=\underline\rho(\mathcal C)^2$ and $\lambda_i=w_i^2$.
All vectors $\mathbf q_i-\mathbf R\mathbf p_i$ lie in one fixed closed ball.
Projecting translations onto this ball cannot increase residuals, so the transformations may be restricted to a compact set $K$.
Let
\begin{equation}
\mathbf c(T)=\bigl(\|T(\mathbf p_i)-\mathbf q_i\|^2\bigr)_{i\in\mathcal C},
\qquad
V=\operatorname{conv}\mathbf c(K).
\label{eq:appendix-residual-hull}
\end{equation}
If $V$ did not intersect the downward box $\{\mathbf z:z_i\le b\}$, strict separation would give nonnegative weights $\lambda_i$ summing to one with
\begin{equation}
\min_{T\in K}\boldsymbol\lambda^\top\mathbf c(T)>b,
\label{eq:appendix-separation-contradiction}
\end{equation}
contradicting the definition of $b$.
Choose a point in this intersection and a finite convex representation
\begin{equation}
\mathbf z=\sum_j\theta_j\mathbf c(T_j),
\qquad
\theta_j>0,
\qquad
\sum_j\theta_j=1.
\label{eq:appendix-finite-fit-combination}
\end{equation}
Then
\begin{equation}
\sum_j\theta_j\|T_j(\mathbf p_i)-\mathbf q_i\|^2\le b\quad\forall i.
\label{eq:appendix-mixture-residual-bound}
\end{equation}
For maximizing weights $\boldsymbol\lambda^*$, every $T_j$ has weighted error at least $b$, while their $\theta_j$-weighted average is at most $b$.
Thus every $T_j$ is a weighted minimizer and aligns the same $\boldsymbol\lambda^*$-weighted centroids $\bar{\mathbf p},\bar{\mathbf q}$.
Write
\begin{equation}
T_j(\mathbf p)=\mathbf U_j(\mathbf p-\bar{\mathbf p})+\bar{\mathbf q},
\qquad
\mathbf U_j\in\mathrm O(d).
\label{eq:appendix-common-centroids}
\end{equation}

Take an SVD and define
\begin{equation}
\mathbf N=\sum_j\theta_j\mathbf U_j
=\mathbf A\boldsymbol\Sigma\mathbf B^\top,
\qquad
\mathbf R=\mathbf A\mathbf B^\top.
\label{eq:appendix-orthogonal-average}
\end{equation}
The singular values $\sigma_j$ of $\mathbf N$ lie in $[0,1]$, so
\begin{equation}
\|2\mathbf N-\mathbf R\|_2
=\max_j|2\sigma_j-1|\le1.
\label{eq:appendix-polar-contraction}
\end{equation}
Hence
\begin{equation}
\mathbf y^\top(2\mathbf N-\mathbf R)\mathbf x
\le\|\mathbf x\|\|\mathbf y\|
\le\frac{\|\mathbf x\|^2+\|\mathbf y\|^2}{2}.
\label{eq:appendix-polar-cross-bound}
\end{equation}
Expanding the squared errors then gives
\begin{equation}
\tfrac12\|\mathbf R\mathbf x-\mathbf y\|^2
\le\sum_j\theta_j\|\mathbf U_j\mathbf x-\mathbf y\|^2.
\label{eq:appendix-averaged-fit-bound}
\end{equation}
Taking $\mathbf x=\mathbf p_i-\bar{\mathbf p}$ and $\mathbf y=\mathbf q_i-\bar{\mathbf q}$, \cref{eq:appendix-mixture-residual-bound} gives $\rho(\mathcal C)^2\le2b$, proving~\cref{eq:weighted-fitting-bound}.
\end{proof}

\begin{proof}[Proof of \cref{lem:local-global-affine}]
Put $\varepsilon=\rho_{d+2}(\mathcal C)$, $L=\sum_i|\alpha_i|$, and $\mathbf x=\sum_i\alpha_i\mathbf p_i$.
Fix a unit direction $\mathbf u$.
We reweight the correspondences to maximize the target projection while preserving the source combination, the zero-sum condition, and the coefficient budget.
Writing $\beta_i=\beta_i^+-\beta_i^-$ gives the linear program
\begin{equation}
\begin{aligned}
\max_{\boldsymbol\beta^+,\boldsymbol\beta^-,s\ge0}\quad
&\mathbf u^\top\sum_i(\beta_i^+-\beta_i^-)\mathbf q_i\\
\text{subject to}\quad
&\sum_i(\beta_i^+-\beta_i^-)\mathbf p_i=\mathbf x,\\
&\sum_i(\beta_i^+-\beta_i^-)=0,\\
&\sum_i(\beta_i^++\beta_i^-)+s=L.
\end{aligned}
\label{eq:appendix-directional-reweighting}
\end{equation}
The positive and negative parts of $\boldsymbol\alpha$, with $s=0$, are feasible.
The feasible set is compact, and the program has $d+2$ scalar equality constraints: $d$ for the source vector, one for the coefficient sum, and one for the budget.
An optimal basic solution therefore has at most $d+2$ nonzero variables.
Its coefficients $\beta_i$ involve at most $d+2$ correspondences and satisfy $\sum_i|\beta_i|\le L$.
Apply \cref{lem:affine-relation} to this smaller set, whose alignment error is at most $\varepsilon$.
Feasibility of $\boldsymbol\alpha$ and optimality of $\boldsymbol\beta$ then give
\begin{equation}
\mathbf u^\top\sum_i\alpha_i\mathbf q_i
\le\mathbf u^\top\sum_i\beta_i\mathbf q_i
\le\left\|\sum_i\beta_i\mathbf q_i\right\|
\le\|\mathbf x\|+\varepsilon\sum_i|\beta_i|
\le\|\mathbf x\|+\varepsilon L.
\label{eq:appendix-local-projection-bound}
\end{equation}
Maximizing over unit directions $\mathbf u$ gives
\begin{equation}
\left\|\sum_i\alpha_i\mathbf q_i\right\|
\le\left\|\sum_i\alpha_i\mathbf p_i\right\|+\varepsilon L.
\label{eq:appendix-local-norm-bound}
\end{equation}
Exchanging source and target gives the reverse bound, proving~\cref{eq:signed-comparison}.
\end{proof}

\begin{proof}[Proof of \cref{lem:few-vector-representation}]
First remove the translation by centering both clouds at
\begin{equation}
\bar{\mathbf p}=\sum_iw_i^2\mathbf p_i,
\qquad
\bar{\mathbf q}=\sum_iw_i^2\mathbf q_i.
\label{eq:appendix-weighted-centroids}
\end{equation}
Form the centered matrices with columns
\begin{equation}
\mathbf F_i=w_i(\mathbf p_i-\bar{\mathbf p}),\qquad
\mathbf G_i=w_i(\mathbf q_i-\bar{\mathbf q}).
\label{eq:weighted-centered-clouds}
\end{equation}
The optimal translation is $\bar{\mathbf q}-\mathbf R\bar{\mathbf p}$~\cite{arun1987least}, giving
\begin{equation}
E(\mathbf w)=\min_{\mathbf R\in\mathrm O(d)}\|\mathbf R\mathbf F-\mathbf G\|_F^2.
\label{eq:weighted-procrustes}
\end{equation}

Only coefficient directions outside the common kernel contribute to this error.
Set
\begin{equation}
S=(\ker\mathbf F\cap\ker\mathbf G)^\perp,
\qquad
r=\dim S\le\operatorname{rank}\mathbf F+\operatorname{rank}\mathbf G\le2d.
\label{eq:appendix-active-coefficient-space}
\end{equation}
If $r=0$, both matrices vanish and the claim is immediate.
Otherwise, the Fuchs--Caves identity~\cite{fuchs1995mathematical}, combined with the Procrustes trace formula, gives an orthonormal basis $\mathbf v_1,\ldots,\mathbf v_r$ of $S$ for which
\begin{equation}
E(\mathbf w)=\sum_{j=1}^r
(\|\mathbf F\mathbf v_j\|-\|\mathbf G\mathbf v_j\|)^2.
\label{eq:appendix-procrustes-length-identity}
\end{equation}
Thus the weighted fitting error is a sum of at most $2d$ squared length differences.
For positive-definite real Gram matrices on $S$, the identity is attained by an eigenbasis of a real symmetric matrix; rescaling removes its trace-normalization requirement.
If either Gram matrix is singular, add $\eta\mathbf I_S$ to both before applying the identity.
As $\eta\downarrow0$, compactness gives a convergent subsequence of orthonormal bases, and continuity recovers the displayed identity.

Finally, $\mathbf F\mathbf w=\mathbf G\mathbf w=0$, so $\mathbf w\in S^\perp$ and every basis vector satisfies $\mathbf v_j^\top\mathbf w=0$.
Setting $\boldsymbol\alpha^{(j)}=\mathbf w\odot\mathbf v_j$ therefore gives
\begin{equation}
\sum_i\alpha_i^{(j)}=0,
\qquad
\mathbf F\mathbf v_j=\sum_i\alpha_i^{(j)}\mathbf p_i,
\qquad
\mathbf G\mathbf v_j=\sum_i\alpha_i^{(j)}\mathbf q_i.
\label{eq:appendix-zero-sum-coefficients}
\end{equation}
The centroid terms cancel because the coefficients sum to zero.
Substitution into~\cref{eq:appendix-procrustes-length-identity} proves the claim.
\end{proof}

\begin{proof}[Proof of \cref{prop:squared-distance}]
Use the centered matrices in~\cref{eq:weighted-centered-clouds}, set $\mathbf H=\mathbf I-\mathbf w\mathbf w^\top$, and define
\begin{equation}
\Delta_{ik}=w_iw_k\bigl(\|\mathbf{p}_i-\mathbf{p}_k\|^2-\|\mathbf{q}_i-\mathbf{q}_k\|^2\bigr).
\label{eq:appendix-weighted-distance-matrix}
\end{equation}
Expanding squared distances gives
\begin{equation}
\mathbf F^\top\mathbf F-\mathbf G^\top\mathbf G
=-\tfrac12\mathbf H\boldsymbol\Delta\mathbf H.
\label{eq:appendix-centered-distance-identity}
\end{equation}
Since $\mathbf H$ is an orthogonal projection and $\|\mathbf w\|_2=1$ gives $\|\boldsymbol\Delta\|_F\le h(\mathcal C)$, the Procrustes bound~\cite{maystre2025procrustes} yields
\begin{equation}
E(\mathbf w)\le\sqrt{2d}\|\mathbf F^\top\mathbf F-\mathbf G^\top\mathbf G\|_F
\le\tfrac{\sqrt{2d}}2\|\boldsymbol\Delta\|_F
\le\tfrac{\sqrt{2d}}2h(\mathcal C).
\label{eq:appendix-weighted-squared-distance-bound}
\end{equation}
Maximizing over $\mathbf w$ and applying \cref{lem:weighted-fitting} gives~\cref{eq:squared-certificate}.
\end{proof}

\subsection{Proper motions and anchor control}
\label{app:proper-motions}
\begin{proof}[Proof of \cref{lem:proper-local-equivalence}]
Because there are at most $d$ source points, they lie in an affine hyperplane.
Reflection $J$ across this hyperplane fixes them all.
If an optimal unrestricted fit $T$ is improper, $T\circ J$ is proper and has the same residuals.
Thus restricting the fit to proper motions does not change the optimal error.
\end{proof}

\begin{proof}[Proof of \cref{lem:anchor-control}]
Use any anchor indices $\mathcal A$ as in the statement.
For the rotations $\mathbf R,\mathbf S\in\mathrm{SO}(d)$ of $T,U$, halving the planar rotation angles gives an orthogonal square root $\mathbf Q^2=\mathbf R^\top\mathbf S$.
The output isometry $\mathbf Q^\top\mathbf R^\top$ preserves disagreement lengths and transforms their difference into
\begin{equation}
f(\mathbf x)=\mathbf Q^\top\mathbf R^\top\bigl(T(\mathbf x)-U(\mathbf x)\bigr)
=\mathbf K\mathbf x+\mathbf b,
\qquad
\mathbf K=\mathbf Q^\top-\mathbf Q,
\qquad
\mathbf K^\top=-\mathbf K.
\label{eq:appendix-skew-disagreement}
\end{equation}

Choose the origin in the anchor hyperplane $H$, let $\mathbf n$ be its unit normal, and write $\mathbf p_i=\mathbf v_i+h_i\mathbf n$ with $\mathbf v_i\in H$.
For $j\in\mathcal A$, let $\beta_j$ be the barycentric coordinate associated with vertex $\mathbf p_j$ in $H$, and let $\mathbf g_j$ be its gradient within $H$.
Barycentric interpolation along $H$ and skew symmetry give
\begin{equation}
\sum_{j\in\mathcal A}\mathbf g_j^\top f(\mathbf p_j)=0,
\qquad
\mathbf K\mathbf n=-\sum_{j\in\mathcal A}\mathbf g_j\mathbf n^\top f(\mathbf p_j).
\label{eq:appendix-normal-disagreement}
\end{equation}
Combining the tangential interpolation with the normal displacement $h_i\mathbf K\mathbf n$ yields
\begin{equation}
f(\mathbf p_i)=\sum_{j\in\mathcal A}\mathbf M_{ij}f(\mathbf p_j),
\qquad
\mathbf M_{ij}=\beta_j(\mathbf v_i)\mathbf I
+h_i(\mathbf n\mathbf g_j^\top-\mathbf g_j\mathbf n^\top).
\label{eq:appendix-anchor-interpolation}
\end{equation}

On the plane spanned by $\mathbf n$ and $\mathbf g_j$, the nonscalar part of $\mathbf M_{ij}$ is skew-symmetric; it vanishes on the orthogonal complement.
This gives the square-root expression below.
Since $\|\mathbf g_j\|$ is the reciprocal of the altitude at $\mathbf p_j$, this expression also equals the volume ratio when that vertex is replaced by $\mathbf p_i$:
\begin{equation}
\|\mathbf M_{ij}\|_2
=\sqrt{\beta_j(\mathbf v_i)^2+h_i^2\|\mathbf g_j\|^2}
=w_{ij}.
\label{eq:appendix-anchor-volume-control}
\end{equation}
Consequently,
\begin{equation}
\|f(\mathbf p_i)\|
\le\sum_{j\in\mathcal A}w_{ij}\|f(\mathbf p_j)\|
\le\kappa\max_{j\in\mathcal A}\|f(\mathbf p_j)\|.
\label{eq:appendix-proper-anchor-conclusion}
\end{equation}
The output isometry preserves these norms, proving~\cref{eq:weighted-anchor-control}; bounding each weight sum by $\kappa$ and taking the maximum over $i$ gives~\cref{eq:anchor-control}.
No volume-maximality assumption was used.
\end{proof}
\section{Sparse certificates from anchor control}
\label{app:sparse-certificates}
We give the simpler bound used in \cref{tab:eval-certificates}, extend anchor control to motions allowing reflections, and describe anchor selection.

\subsection{A simpler anchor bound}
\label{app:one-fit-certificates}
For the anchors and reference residuals of \cref{sec:sparse-certificates}, suppose $\kappa=\max_i\sum_{j\in\mathcal A}w_{ij}\le d$ and set $e_{\mathcal A}=\max_{j\in\mathcal A}e_j$.
If a proper motion fits $\mathcal A\cup\{i\}$ within $b$, anchor control gives
\begin{equation}
e_i\le\sum_{j\in\mathcal A}w_{ij}(e_j+b)+b
\le d(e_{\mathcal A}+b)+b.
\end{equation}
Rearranging and taking the infimum over feasible $b$ yields
\begin{equation}
\ell_i=\frac{[e_i-de_{\mathcal A}]_+}{1+d}
\le\rho^+(\mathcal A\cup\{i\}).
\label{eq:anchor-fit-lower}
\end{equation}
Thus $\ell_i>\tau$ proves that the anchors and match $i$ cannot be aligned together within $\tau$.
Retaining the individual geometric weights gives the stronger bound in \cref{eq:anchor-cheap-lower}.

\subsection{Anchor control allowing reflections}
\label{app:unrestricted-anchors}
For a source cloud of affine dimension $r$, choose $r+1$ affinely independent anchors spanning its affine hull and write
\begin{equation}
\mathbf p_i=\sum_{j\in\mathcal A}\lambda_{ij}\mathbf p_j,
\qquad \sum_{j\in\mathcal A}\lambda_{ij}=1,
\qquad w_{ij}=|\lambda_{ij}|.
\label{eq:unrestricted-anchor-coordinates}
\end{equation}
These weights are the corresponding replacement-volume ratios.
The difference of any two rigid motions is affine, so any $T,U\in\mathrm E(d)$ satisfy
\begin{equation}
\|T(\mathbf p_i)-U(\mathbf p_i)\|
\le\sum_{j\in\mathcal A}w_{ij}
\|T(\mathbf p_j)-U(\mathbf p_j)\|.
\label{eq:weighted-anchor-control-general}
\end{equation}
The argument in \cref{sec:sparse-certificates} therefore gives \cref{eq:anchor-cheap-lower,eq:anchor-fit-bracket} with $\rho$ replacing $\rho^+$ and $\kappa_i=\sum_{j\in\mathcal A}w_{ij}$.
When $|\lambda_{ij}|\le1$, we also have $\kappa_i\le r+1\le d+1$, giving \cref{eq:anchor-fit-lower} with $d$ replaced by $r+1$ and $\rho^+$ by $\rho$.

\subsection{Anchor selection}
\label{app:anchor-selection}
Starting from a nondegenerate source simplex, replace a vertex by a point from the tested set whenever this increases the volume, stopping when no such replacement exists.
Every replacement-volume ratio is then at most one, giving $\kappa\le d$ for the proper-motion anchors in \cref{lem:anchor-control}, or $\kappa\le r+1$ for anchors spanning an $r$-dimensional affine hull.
Strict increases terminate because there are finitely many candidate simplices.
If selection stops earlier or additional points are tested, use their actual geometric weights rather than assuming these bounds on $\kappa$.

\section{Weighted least-squares anchor certificates}
\label{app:weighted-anchor-svd}
We derive a spectral expression for the sensitivity in~\cref{eq:anchor-svd-sensitivity}, prove \cref{prop:weighted-anchor-svd}, and extend the certificate to motions allowing reflections.
Throughout this section, let $G\in\{\mathrm{SE}(d),\mathrm E(d)\}$ and allow any number $m=|\mathcal A|$ of anchors; the weighted loss, centroids, and residuals use sums over indices $j\in\mathcal A$.
Write $E_G(\mathbf w)$ for the optimal weighted fitting error in class $G$, so $E_{\mathrm{SE}(d)}=E^+$ and $E_{\mathrm E(d)}=E$.
The weights are fixed throughout each test, even when they were chosen for a particular query point.

\subsection{Weighted centering and curvature}
\label{app:svd-curvature}
Fix weights $\mathbf w=(w_j)_{j\in\mathcal A}$ with $\|\mathbf w\|_2=1$.
Define the weighted centroids
$\bar{\mathbf p}_{\mathbf w}=\sum_{j\in\mathcal A}w_j^2\mathbf p_j$ and
$\bar{\mathbf q}_{\mathbf w}=\sum_{j\in\mathcal A}w_j^2\mathbf q_j$.
Write $\mathbf x_j=\mathbf p_j-\bar{\mathbf p}_{\mathbf w}$ and
$\mathbf y_j=\mathbf q_j-\bar{\mathbf q}_{\mathbf w}$.
A minimizer attaining $E_G(\mathbf w)$ has orthogonal factor $\mathbf R_{\mathbf w}$ and translation
$\bar{\mathbf q}_{\mathbf w}-\mathbf R_{\mathbf w}\bar{\mathbf p}_{\mathbf w}$.
Set
\begin{equation}
\mathbf C_{\mathbf w}
=\sum_{j\in\mathcal A}w_j^2\mathbf y_j\mathbf x_j^\top,
\qquad
\mathbf B_{\mathbf w}
=\mathbf R_{\mathbf w}^\top\mathbf C_{\mathbf w}.
\label{eq:weighted-svd-geometry}
\end{equation}
At an exact optimum, $\mathbf B_{\mathbf w}$ is symmetric.
Choose an orthogonal eigenbasis $\mathbf V$, ordered by eigenvalues $b_1\le\cdots\le b_d$.
For unrestricted motions, the $b_r$ are the singular values of $\mathbf C_{\mathbf w}$.
For proper motions, the determinant correction can change the sign of one smallest singular value.
The curvature conditions used here are
\begin{equation}
\begin{cases}
b_r+b_s>0\quad\text{for every }r<s, & G=\mathrm{SE}(d),\\
b_r>0\quad\text{for every }r, & G=\mathrm E(d).
\end{cases}
\label{eq:weighted-svd-curvature}
\end{equation}
A proper rotation moves in two-dimensional planes, so its cost involves pairs of eigenvalues.
For $d\ge2$, positivity of all pair sums is equivalent to~\cref{eq:proper-svd-curvature}.
A reflection can reverse a single direction, which explains the stronger condition for unrestricted motions.
These are sufficient conditions for this certificate; failure of a condition does not imply infeasibility.

For a query point, put $\mathbf x_i=\mathbf p_i-\bar{\mathbf p}_{\mathbf w}$ and $\boldsymbol\xi_i=\mathbf V^\top\mathbf x_i$.
Let $\mathbf e_r$ denote the $r$th coordinate vector and define
\begin{equation}
\mathbf h_{i,rs}=\xi_{i,s}\mathbf e_r-\xi_{i,r}\mathbf e_s,
\qquad
\mathbf M_i=\sum_{r<s}
\frac{\mathbf h_{i,rs}\mathbf h_{i,rs}^\top}{b_r+b_s}.
\label{eq:weighted-svd-plane-matrix}
\end{equation}
Under the corresponding curvature condition, use
\begin{equation}
\chi_i(\mathbf w)^2=
\begin{cases}
1+\lambda_{\max}(\mathbf M_i), & G=\mathrm{SE}(d),\\
1+\mathbf x_i^\top\mathbf B_{\mathbf w}^{-1}\mathbf x_i,
& G=\mathrm E(d).
\end{cases}
\label{eq:weighted-svd-coefficient}
\end{equation}
Only the $d\times d$ matrix $\mathbf M_i$ is needed; there is no need to form a matrix with one row for every rotation plane.
In dimension one the sum is empty, giving $\chi_i=1$ for proper motions, which are translations.

For proper motions in $d\ge2$, this coefficient equals the Rayleigh-quotient definition~\cref{eq:anchor-svd-sensitivity}.
Indeed, write $\omega_{rs}$ for the upper-triangular entries of $\mathbf V^\top\boldsymbol\Omega\mathbf V$.
Skew symmetry gives
\begin{equation}
\operatorname{tr}(\boldsymbol\Omega^\top\boldsymbol\Omega\mathbf B_{\mathbf w})
=\sum_{r<s}(b_r+b_s)\omega_{rs}^2,
\qquad
\mathbf V^\top\boldsymbol\Omega\mathbf x_i
=\sum_{r<s}\omega_{rs}\mathbf h_{i,rs}.
\label{eq:svd-rayleigh-coordinates}
\end{equation}
The change of variables $z_{rs}=\sqrt{b_r+b_s}\,\omega_{rs}$ turns the denominator into $\|\mathbf z\|^2$.
The numerator becomes $\|\mathbf L_i\mathbf z\|^2$, where the columns of $\mathbf L_i$ are $\mathbf h_{i,rs}/\sqrt{b_r+b_s}$.
Since $\mathbf L_i\mathbf L_i^\top=\mathbf M_i$, the largest quotient is $\lambda_{\max}(\mathbf M_i)$.
This proves the equivalence without weakening the coefficient.

For $G=\mathrm E(d)$, \cref{prop:weighted-anchor-svd} holds with the positive-definiteness condition and coefficient in~\crefrange{eq:weighted-svd-curvature}{eq:weighted-svd-coefficient}.
The following argument proves both statements.

\subsection{The displacement bound and rejection test}
\label{app:svd-displacement}
\begin{proof}[Proof of \cref{prop:weighted-anchor-svd} and its unrestricted-motion extension]
Let $T(\mathbf p)=\mathbf R\mathbf p+\mathbf t$ be a competing pose in $G$.
Write $\mathbf Q=\mathbf R_{\mathbf w}^\top\mathbf R$,
$\mathbf D=\mathbf Q-\mathbf I$, and
$\mathbf z=\mathbf t+\mathbf R\bar{\mathbf p}_{\mathbf w}-\bar{\mathbf q}_{\mathbf w}$.
Weighted centering gives
\begin{equation}
\sum_{j\in\mathcal A}w_j^2\|T(\mathbf p_j)-\mathbf q_j\|^2
=\sum_{j\in\mathcal A}w_j^2\|\mathbf R\mathbf x_j-\mathbf y_j\|^2+\|\mathbf z\|^2.
\end{equation}
The rotation-dependent term is $-2\operatorname{tr}(\mathbf R^\top\mathbf C_{\mathbf w})$.
Using symmetry of $\mathbf B_{\mathbf w}$ and
$\mathbf D^\top\mathbf D=2\mathbf I-\mathbf Q-\mathbf Q^\top$ therefore yields the exact identity
\begin{equation}
\sum_{j\in\mathcal A}w_j^2\|T(\mathbf p_j)-\mathbf q_j\|^2-E_G(\mathbf w)
=\underbrace{\operatorname{tr}(\mathbf D^\top\mathbf D\mathbf B_{\mathbf w})}_{u^2}
+\|\mathbf z\|^2.
\label{eq:weighted-svd-loss-identity}
\end{equation}
We next bound $\|\mathbf D\mathbf x_i\|$, the rotational displacement of the query point.

For $G=\mathrm{SE}(d)$, the relative matrix $\mathbf Q$ is a proper rotation.
Decompose it into orthogonal rotation planes and take a square root $\mathbf H$ by halving each plane's angle.
This also works for angle $\pi$; the $-1$ eigenspace has even dimension.
Then $\mathbf H$ is orthogonal and
\begin{equation}
\boldsymbol\Omega=\mathbf H-\mathbf H^\top,
\qquad
\mathbf D=\mathbf H\boldsymbol\Omega,
\qquad
\mathbf D^\top\mathbf D=\boldsymbol\Omega^\top\boldsymbol\Omega.
\end{equation}
In particular, $\boldsymbol\Omega$ is skew-symmetric and
$\|\mathbf D\mathbf x_i\|=\|\boldsymbol\Omega\mathbf x_i\|$.
For $d\ge2$, the definition~\cref{eq:anchor-svd-sensitivity} therefore gives
\begin{equation}
\|\mathbf D\mathbf x_i\|^2
\le\bigl(\chi_i(\mathbf w)^2-1\bigr)
\operatorname{tr}(\boldsymbol\Omega^\top\boldsymbol\Omega\mathbf B_{\mathbf w})
=\bigl(\chi_i(\mathbf w)^2-1\bigr)u^2.
\end{equation}
This argument uses finite rotations, not a small-angle approximation.
For $d=1$, the rotational displacement vanishes and the same inequality holds with $\chi_i=1$.

For $G=\mathrm E(d)$, $\mathbf Q$ may also be a reflection.
Positive definiteness of $\mathbf B_{\mathbf w}$ gives the bound directly:
\begin{equation}
\|\mathbf D\mathbf x_i\|^2
=\|\mathbf D\mathbf B_{\mathbf w}^{1/2}
\mathbf B_{\mathbf w}^{-1/2}\mathbf x_i\|^2
\le\|\mathbf D\mathbf B_{\mathbf w}^{1/2}\|_F^2
\|\mathbf B_{\mathbf w}^{-1/2}\mathbf x_i\|^2
=\bigl(\chi_i(\mathbf w)^2-1\bigr)u^2.
\label{eq:appendix-unrestricted-svd-query}
\end{equation}
No restriction on the determinant is needed in this step.

For either class, the triangle inequality and Cauchy--Schwarz now imply
\begin{equation}
\|T(\mathbf p_i)-T_{\mathbf w}(\mathbf p_i)\|
\le\sqrt{\chi_i(\mathbf w)^2-1}\,u+\|\mathbf z\|
\le\chi_i(\mathbf w)\sqrt{u^2+\|\mathbf z\|^2}.
\label{eq:appendix-svd-query-combination}
\end{equation}
If every anchor residual of $T$ is at most $\tau$, normalization of the squared weights implies $\sum_{j\in\mathcal A}w_j^2\|T(\mathbf p_j)-\mathbf q_j\|^2\le\tau^2$.
Substitution in~\cref{eq:weighted-svd-loss-identity} proves~\cref{eq:weighted-svd-displacement}.
\end{proof}

If $E_G(\mathbf w)>\tau^2$, no pose in $G$ can fit the anchors at that tolerance.
Otherwise, a pose fitting both the anchors and match $i$ must satisfy
$e_i(\mathbf w)\le\tau+\chi_i(\mathbf w)\sqrt{\tau^2-E_G(\mathbf w)}$,
by the residual triangle inequality.
This proves~\cref{eq:anchor-svd-rejection}.
Each choice of weights gives a necessary condition for the same feasible poses, so any one failed condition rejects the joint set.

\subsection{The 3D formula}
\label{app:svd-3d}
In three dimensions, a skew-symmetric matrix is represented by its axial vector.
In these coordinates the curvature and point-displacement matrices become
\begin{equation}
\mathbf J_{\mathbf w}
=\operatorname{tr}(\mathbf B_{\mathbf w})\mathbf I-\mathbf B_{\mathbf w},
\qquad
\mathbf K_i=\|\mathbf x_i\|^2\mathbf I-\mathbf x_i\mathbf x_i^\top,
\qquad
\chi_i^2=1+\lambda_{\max}\!\left(
\mathbf J_{\mathbf w}^{-1/2}\mathbf K_i\mathbf J_{\mathbf w}^{-1/2}\right).
\label{eq:weighted-svd-3d}
\end{equation}
The eigenvalues of $\mathbf J_{\mathbf w}$ are precisely the three pair sums $b_r+b_s$.
The two expressions for $\chi_i$ in the proper-motion case agree, since the nonzero eigenvalues of $\mathbf L\mathbf L^\top$ and $\mathbf L^\top\mathbf L$ agree for the corresponding weighted displacement map $\mathbf L$.
Thus~\cref{eq:anchor-svd-sensitivity} preserves the coefficient used in the 3D experiments.
The matrix $\operatorname{tr}(\mathbf B)\mathbf I-\mathbf B$ alone is not the corresponding rotation curvature in higher dimensions.

\subsection{Geometry-dependent fitting weights}
\label{app:svd-fitting-weights}
Uniform fitting weights $w_j=1/\sqrt d$ allow one reference transformation to be reused for all points.
For match $i$, we also consider fitting weights based on the geometric coefficients:
\begin{equation}
\mathbf w^{(i)}=\left(\sqrt{\frac{w_{ij}}{\kappa_i}}\right)_{j\in\mathcal A}.
\label{eq:geometry-fitting-weights}
\end{equation}
Their squares sum to one and give more weight to anchors whose movement has greater influence at $\mathbf p_i$.
The reference transformation, residual, and sensitivity are recomputed for these weights.

The geometric weights $w_{ij}$ in~\cref{eq:weighted-anchor-control-general} are displacement coefficients and need not sum to one.
To relate these coefficients to the normalized fitting weights, let $d_j\ge0$ be the anchor displacements and suppose $w_j^2>0$ wherever $w_{ij}>0$.
Cauchy--Schwarz gives
\begin{equation}
\left(\sum_{j\in\mathcal A}w_{ij}d_j\right)^2
\le\left(\sum_{\substack{j\in\mathcal A\\w_{ij}>0}}\frac{w_{ij}^2}{w_j^2}\right)
\left(\sum_{j\in\mathcal A}w_j^2d_j^2\right).
\end{equation}
For fixed geometric coefficients, the first factor is minimized by $w_j^2=w_{ij}/\kappa_i$, with zero weight where $w_{ij}=0$.
This does not establish optimality for the weighted anchor-fit certificate, whose pose, residual, and curvature also change with the weights.
The rule uses source geometry alone and has no mixing parameter.

\section{Proof of maximum clique hardness}
\label{app:computational-proofs}
\label{app:hardness}
The input is a supplied list of correspondences
$(\mathbf p_i,\mathbf q_i)\in\mathbb Q^d\times\mathbb Q^d$, with coordinates encoded in binary.
There is no search for an unknown correspondence permutation.
For clique selection, two indices are adjacent precisely when
\begin{equation}
\delta_{ij}=
\left|\|\mathbf p_i-\mathbf p_j\|-\|\mathbf q_i-\mathbf q_j\|\right|
\le1.
\end{equation}
All complexity assertions below keep the ambient dimension $d$ fixed and refer to binary input size.

Here $n=n_H$ and $m=m_H$ denote the vertex and edge counts of the simple input graph $H$.
Number the vertices and the edges separately, and orient each edge as $e=(u,v)$.
The orientation only names its two new vertices; it does not change the undirected graph.
Replace $uv$ by the three-edge path $u-a_e-b_e-v$ to obtain the 2-subdivision $S_2(H)$.
Let $F(H)$ be its complement.
It has $n+2m$ vertices, and its only nonedges are the three path pairs $ua_e,a_eb_e,b_ev$ for each original edge.
We first realize exactly these nonedges geometrically, then prove the cardinality identity needed for the reduction.

\begin{lemma}[Planar representation]
Every $F(H)$ has a tolerance-1 planar consistency representation with distinct source points and distinct target points, computable in polynomial time with rational coordinates of $O(\log(n_H+m_H+1))$ bits per coordinate.
\label{lem:realization}
\end{lemma}

\paragraph{Rational directions and one small parameter}
Let $N=\max\{n,m,1\}$ and, for $1\le j\le N$, choose rational unit directions on a short arc:
\begin{equation}
t_j=\frac{j}{4(N+1)},\qquad
\mathbf c_j=\left(\frac{1-t_j^2}{1+t_j^2},
\frac{2t_j}{1+t_j^2}\right).
\label{eq:circle-labels}
\end{equation}
The labels $\mathbf c_w$ and $\mathbf c_e$ specify vertex and edge directions, respectively.
These are separate uses of the same direction list; a vertex and an edge need not have different labels.
The formula for $\mathbf c_j$ is the rational parametrization of the unit circle, with $0<t_j<1/4$.
The exact chord identity
\begin{equation}
\norm{\mathbf c_i-\mathbf c_j}
=\frac{2|t_i-t_j|}{\sqrt{(1+t_i^2)(1+t_j^2)}}
\label{eq:circle-chord}
\end{equation}
shows that every chord is shorter than $1/2$ and distinct directions are separated by at least $1/[3(N+1)]$.
Indeed, $|t_i-t_j|\ge1/[4(N+1)]$ for distinct indices, and the denominator in~\cref{eq:circle-chord} is at most $17/16$.
The deliberately weaker bound $1/[3(N+1)]$ will suffice.
Choose the radial inflation from this separation:
\begin{equation}
\begin{gathered}
\Delta=\min\left(\{1/4\}\cup
\{\norm{\mathbf c_i-\mathbf c_j}^2:1\le i<j\le N\}\right),
\qquad \varepsilon=\frac{\Delta}{32},\\
R=\frac12+\varepsilon,\qquad A=\frac32+2\varepsilon^2,
\qquad s_e=\varepsilon+\frac{e}{m+1}\varepsilon^2.
\end{gathered}
\label{eq:parameters}
\end{equation}
Thus $0<\varepsilon\le1/128$ and $\varepsilon<s_e<\varepsilon+\varepsilon^2$.
The minimum defining $\Delta$ includes $1/4$, so it is defined even when there is only one direction.
When there are gadgets, the edge-dependent term in $s_e$ separates their target radii.
Use the correspondences
\begin{equation}
\begin{array}{c|cc}
\text{vertex}&\mathbf p&\mathbf q\\\hline
w&\varepsilon^2\mathbf c_w&A\mathbf c_w\\
a_e&R\mathbf c_e&-s_e\mathbf c_u\\
b_e&-R\mathbf c_e&-s_e\mathbf c_v.
\end{array}
\label{eq:coordinates}
\end{equation}
All gadget sources lie on one circle.
The order-$\varepsilon^2$ offsets move original sources onto a small circle, move original targets slightly outward, and give each edge a distinct gadget-target radius.
The order-$\varepsilon$ radial inflation makes each matched source diameter exceed one.
The angular separation controls unmatched pairs, while the smaller offsets give a strict margin and remove repeated points.
The verification below keeps these roles separate.

\begin{proof}[Proof of \cref{lem:realization}]
We verify that every edge of $F(H)$ has discrepancy at most $1-\varepsilon^2$, and every nonedge has discrepancy at least $1+\varepsilon^2$.
There are four pair types, which exhaust all unordered pairs of correspondences.

\emph{Two originals.}
Their source and target distances are respectively
$\varepsilon^2\|\mathbf c_w-\mathbf c_z\|$ and
$A\|\mathbf c_w-\mathbf c_z\|$.
Their discrepancy is
$(A-\varepsilon^2)\norm{\mathbf c_w-\mathbf c_z}
\le3/4+\varepsilon^2/2<1-\varepsilon^2$.

\emph{An original and a gadget.}
For a gadget assigned to $u$ and an original $w$, let
$T=\norm{A\mathbf c_w+s_e\mathbf c_u}$ be their target distance and $S$ their source distance.
The latter lies in $[R-\varepsilon^2,R+\varepsilon^2]$.
Since $T-S\ge A-s_e-R-\varepsilon^2\ge1-2\varepsilon>0$, the discrepancy is $T-S$.
If $w=u$, it is at least
$A+s_e-R-\varepsilon^2\ge1+\varepsilon^2$.
Thus the designated endpoint is inconsistent with its gadget.
For $w\ne u$, the strict triangle-inequality deficit satisfies
\begin{equation}
A+s_e-T
=\frac{As_e\norm{\mathbf c_w-\mathbf c_u}^2}{A+s_e+T}
\ge\frac{s_e\Delta}{3}
\ge\frac{32}{3}\varepsilon^2.
\label{eq:circle-deficit}
\end{equation}
Here $T\le A+s_e$ and $A\ge2s_e$ give the first inequality.
The numerator follows from
$(A+s_e)^2-T^2=As_e\|\mathbf c_w-\mathbf c_u\|^2$.
Consequently
$T-S\le1+4\varepsilon^2-(32/3)\varepsilon^2<1-\varepsilon^2$.
This uses $S\ge R-\varepsilon^2$ and
$A+s_e-R+\varepsilon^2\le1+4\varepsilon^2$.
The same argument applies to $b_e$ and its assigned endpoint $v$.

\emph{Matched gadgets.}
Their source distance is $2R=1+2\varepsilon$, and their target distance is at most $s_e/2$.
The latter bound follows because the target directions are the distinct endpoints $\mathbf c_u,\mathbf c_v$ on the short arc.
Thus their discrepancy is at least
$1+(3/2)\varepsilon-\varepsilon^2/2>1+\varepsilon^2$.

\emph{Unmatched gadgets.}
Every gadget target distance is at most $2\varepsilon+2\varepsilon^2$.
For two gadgets of the same type, the source distance is at most $R/2$, so both distances are below $1/2$.
For opposite types from distinct edges $e,f$, the source distance $B$ satisfies
\begin{equation}
B^2
=R^2\norm{\mathbf c_e+\mathbf c_f}^2
=(1+2\varepsilon)^2\left(1-\frac{\norm{\mathbf c_e-\mathbf c_f}^2}{4}\right)
\le(1+2\varepsilon)^2-\frac{\Delta}{4}
=(1-2\varepsilon)^2.
\label{eq:opposite-gadgets}
\end{equation}
Both distances are therefore at most $1-2\varepsilon$, so their absolute difference is also at most $1-2\varepsilon<1-\varepsilon^2$.
In particular, no lower bound on their target distance is needed: bounding each distance above already bounds their difference.

Exactly the pairs $ua_e,a_eb_e,b_ev$ are inconsistent, as required.

\emph{Distinct coordinates and encoding.}
All source points are distinct: original directions are distinct on radius $\varepsilon^2$, and gadget directions are distinct on the two opposite arcs of radius $R$.
All target points are distinct: originals have distinct directions on radius $A$, gadgets from different edges have distinct radii $s_e$, and the two gadgets of one edge have different endpoint directions.
The radius ranges are disjoint.

The formulas also cover edgeless graphs and the empty instance.
Each squared chord in~\cref{eq:parameters} is a rational with $O(\log(N+1))$ bits, and $\Delta$ is the minimum of $O(N^2)$ such values and $1/4$.
Taking a minimum selects one of these rationals; it does not multiply their denominators together.
Computing this minimum and the remaining rational formulas takes polynomial time and gives $O(\log(N+1))$ bits per coordinate.
Since $N\le n+m+1$, this proves the stated encoding bound.
Also, the chord separation gives
$\Delta\ge1/[9(N+1)^2]$.
Thus the strict discrepancy margin $\varepsilon^2$ is at least an inverse polynomial in $N$.
It is not a fixed margin independent of instance size.
The coordinates themselves remain bounded: all radii are at most $A=3/2+2\varepsilon^2<1.501$.
\end{proof}

\paragraph{Exact edge verification}
Put $A_{ij}=\norm{\mathbf{p}_i-\mathbf{p}_j}^2$ and $B_{ij}=\norm{\mathbf{q}_i-\mathbf{q}_j}^2$.
For nonnegative $A,B$, the exact equivalence
\begin{equation}
\begin{gathered}
C=A+B-1,\\
|\sqrt A-\sqrt B|\le1
\iff C\le0\ \text{or}\ C^2-4AB\le0
\end{gathered}
\label{eq:polynomial-test}
\end{equation}
follows by writing $C\le2\sqrt{AB}$ and separating the case $C\le0$ before squaring.
This case separation is necessary: squaring a negative left side would not give an equivalent condition.
The test uses only rational additions, multiplications, and comparisons, so exact threshold ties require no floating-point square roots.

\begin{proof}[Proof of \cref{prop:clique-hardness}]
\emph{Step 1: relate independent sets before and after subdivision.}
Fix the original vertices $S$ selected by an independent set of $S_2(H)$.
The original vertices are pairwise nonadjacent in $S_2(H)$, so any $S$ is initially allowed.
In the gadget $u-a_e-b_e-v$, at most one of $a_e,b_e$ can be selected.
If both $u,v$ belong to $S$, neither can be selected.
If just one endpoint belongs to $S$, choose the gadget vertex adjacent to the other endpoint.
If neither belongs to $S$, choose either gadget vertex.
Different gadgets have no edges between their new vertices, so these choices are independent.
The largest extension of $S$ therefore has size
\begin{equation}
|S|+m-|E(H[S])|.
\end{equation}
Deleting at most one endpoint per edge of $H[S]$ leaves an independent set of $H$ of size at least $|S|-|E(H[S])|$.
It follows that every extension has size at most $m+\alpha(H)$.
Taking $S$ to be a maximum independent set of $H$ attains that bound.
Finally, a clique in a complement is an independent set in the original graph, so
$\omega(F(H))=\alpha(S_2(H))=m+\alpha(H)$.

\emph{Step 2: perform the reduction.}
Reduce from the NP-complete independent-set problem~\cite{karp1972reducibility}.
Given $(H,k)$, construct the rational planar representation from \cref{lem:realization} and ask for a clique of size at least $m+k$.
The identity just proved gives an affirmative answer exactly when $\alpha(H)\ge k$.
There are $n+2m$ correspondences, and the construction and its output bit length are polynomial, establishing NP-hardness.
The construction has distinct points in each cloud, so this reduction also applies when the supplied correspondences are one-to-one.

\emph{Step 3: verify membership and extend the dimension.}
For rational coordinates, the squared distances are rational and~\cref{eq:polynomial-test} checks each edge using exact rational arithmetic.
A proposed subset of at least $m+k$ indices is a polynomial-size witness.
Checking all its pairs verifies it in polynomial time, proving membership in NP.
Appending zeros preserves all distances and establishes the result for every fixed $d\ge2$.
Scaling both point sets gives any fixed positive rational tolerance.
\end{proof}

\end{document}

%% file: macros.tex
\usepackage{amsmath} %
\usepackage{amssymb} %
\usepackage{amsfonts}
\usepackage{makecell}
\usepackage{varwidth}

\usepackage{graphicx}
\usepackage{tabularx}
\usepackage{multirow}
\usepackage{mathtools}
\usepackage{esvect} %
\usepackage{siunitx}
\usepackage{cite}

\usepackage[T1]{fontenc} %
\usepackage[nomath]{lmodern}

\usepackage{mathtools} %

\usepackage{booktabs} %
\usepackage{array}
\usepackage[export]{adjustbox} %

\newcolumntype{x}[1]{%
>{\raggedleft\hspace{0pt}}p{#1}}%

\usepackage{algorithm}
\usepackage[noend]{algpseudocode} %
\definecolor{commentclr}{RGB}{34, 139, 34}

\usepackage{tikz}

\usetikzlibrary{positioning, fit, shapes.geometric, calc,arrows.meta}

\usepackage{cleveref}
\usepackage[caption=false,font=footnotesize]{subfig}

\newsiamremark{remark}{Remark}
\newsiamthm{prop}{Proposition}
\newsiamthm{conjecture}{Conjecture}
\newsiamremark{defin}{Definition}
\newsiamremark{assump}{Assumption}
\newsiamremark{example}{Example}
\newsiamremark{problem}{Problem}

\usepackage{letltxmacro}
\LetLtxMacro\orgvdots\vdots
\LetLtxMacro\orgddots\ddots

\makeatletter
\DeclareRobustCommand\vdots{%
	\mathpalette\@vdots{}%
}
\newcommand*{\@vdots}[2]{%
	\sbox0{$#1\cdotp\cdotp\cdotp\m@th$}%
	\sbox2{$#1.\m@th$}%
	\vbox{%
		\dimen@=\wd0 %
		\advance\dimen@ -3\ht2 %
		\kern.5\dimen@
		\dimen@=\wd2 %
		\advance\dimen@ -\ht2 %
		\dimen2=\wd0 %
		\advance\dimen2 -\dimen@
		\vbox to \dimen2{%
			\offinterlineskip
			\copy2 \vfill\copy2 \vfill\copy2 %
		}%
	}%
}
\DeclareRobustCommand\ddots{%
	\mathinner{%
		\mathpalette\@ddots{}%
		\mkern\thinmuskip
	}%
}
\newcommand*{\@ddots}[2]{%
	\sbox0{$#1\cdotp\cdotp\cdotp\m@th$}%
	\sbox2{$#1.\m@th$}%
	\vbox{%
		\dimen@=\wd0 %
		\advance\dimen@ -3\ht2 %
		\kern.5\dimen@
		\dimen@=\wd2 %
		\advance\dimen@ -\ht2 %
		\dimen2=\wd0 %
		\advance\dimen2 -\dimen@
		\vbox to \dimen2{%
			\offinterlineskip
			\hbox{$#1\mathpunct{.}\m@th$}%
			\vfill
			\hbox{$#1\mathpunct{\kern\wd2}\mathpunct{.}\m@th$}%
			\vfill
			\hbox{$#1\mathpunct{\kern\wd2}\mathpunct{\kern\wd2}\mathpunct{.}\m@th$}%
		}%
	}%
}
\makeatother

\let\oldnl\nl%
\newcommand{\nonl}{\renewcommand{\nl}{\let\nl\oldnl}}%
\makeatother

\def\sl{\mathcal{L}}

\def\ker{\text{ker}}

\graphicspath{{figures/}}

\DeclarePairedDelimiter\norm{\lVert}{\rVert}%
\usepackage[percent]{overpic}

\algblockdefx{RetStruct}{EndRetStruct}[1]{\textbf{return} \textbf{struct} \texttt{#1}}{\textbf{end struct}}

\algblockdefx{Struct}{EndStruct}[1]{\textbf{struct} \texttt{#1}}{\textbf{end struct}}

\crefname{subsection}{Section}{Sections}
\Crefname{subsection}{Section}{Sections}

\AddToHook{env/example/begin}{\crefalias{theorem}{example}}
\crefname{example}{Example}{Examples}
\Crefname{example}{Example}{Examples}

\definecolor{ResultRed}{HTML}{F4CCCC}
\definecolor{ResultYellow}{HTML}{FFF2CC}
\definecolor{ResultGreen}{HTML}{D9EAD3}